\documentclass[11pt]{article}
\usepackage[margin=1in]{geometry}

\usepackage[utf8]{inputenc}
\usepackage{amsmath, amsfonts, amssymb, amsthm, dsfont}
\usepackage{babel}
\usepackage{parskip}
\usepackage{enumerate}
\usepackage[normalem]{ulem} 
\usepackage{graphicx}
\usepackage{mathtools}
\usepackage{xspace}
\usepackage{comment}
\usepackage{bm}
\usepackage{soul}
\usepackage{booktabs}

\usepackage[usenames,dvipsnames]{xcolor}
\usepackage[thmtools-compat]{keytheorems}
\usepackage[backref=page,linktocpage=true,breaklinks,colorlinks,citecolor=PineGreen,linkcolor=Violet]{hyperref}

\usepackage{algorithm}
\usepackage[noend]{algpseudocode}
\usepackage{cleveref}

\DeclarePairedDelimiter\floor{\lfloor}{\rfloor}

\newcommand{\cF}{\mathcal{F}}
\newcommand{\cG}{\mathcal{G}}
\newcommand{\cM}{\mathcal{M}}
\newcommand{\R}{\mathbb{R}}
\newcommand{\Rp}{\mathbb{R}_+}
\newcommand{\N}{\mathbb{N}}
\newcommand{\E}{\mathbb{E}}

\newcommand{\ind}{\mathds{1}}

\newcommand{\Opt}{\textnormal{OPT}}

\newcommand{\Alg}{\textnormal{Alg}}

\newcommand{\ones}{\bm{1}}

\renewcommand{\d}{\mathrm{d}}

\newcommand{\tsum}{{\textstyle \sum}}
\newcommand{\cD}{\mathcal{D}}
\newcommand{\up}[1]{^{(#1)}}

\newtheorem{theorem}{Theorem}[section]
\newtheorem{lemma}[theorem]{Lemma}

\newtheorem{corollary}[theorem]{Corollary}

\newtheorem{proposition}[theorem]{Proposition}
\newtheorem{claim}[theorem]{Claim}
\newtheorem{cor}[theorem]{Corollary}
\newtheorem{definition}[theorem]{Definition}

\newtheorem{fact}[theorem]{Fact}
\newtheorem{remark}[theorem]{Remark}

\newcommand{\vertm}[1]{{\left\vert\kern-0.25ex\left\vert\kern-0.25ex\left\vert #1     \right\vert\kern-0.25ex\right\vert\kern-0.25ex\right\vert}}

\newcommand{\topk}{\ensuremath{\textrm{Top-k}}\xspace}

\newcommand{\ip}[2]{\langle #1, #2 \rangle}

\newcounter{note}[section]

\newcommand{\OPT}{\textup{OPT}}
\newcommand{\probmatch}{{posterior-matching}}
\newcommand{\Probmatch}{{Posterior-matching}}
\newcommand{\ProbMatch}{{Posterior-Matching}}

\newcommand{\cost}{\mathrm{cost}}

\newcommand{\cX}{\mathcal{X}}
\newcommand{\cY}{\mathcal{Y}}

\newcommand{\e}{\varepsilon}

\begin{document} 

\title{Online Algorithms via Minimax and Posterior Matching}

    	\author{Thomas Kesselheim\thanks{
	(thomas.kesselheim@uni-bonn.de)
 Institute of Computer Science,	    University of Bonn. 
    }
	\and Marco Molinaro\thanks{
         (mmolinaro@microsoft.com)
         Microsoft Research and PUC-Rio.   
Supported in part by the Coordenação de Aperfeiçoamento de Pessoal de Nível Superior - Brasil (CAPES) - Finance Code 001, and by Bolsa de Produtividade em Pesquisa $\#3$02121/2025-0 from CNPq.
    }
    \and Kalen Patton\thanks{
        (kpatton33@gatech.edu)
        School of Mathematics,
        Georgia Tech.
        Supported in part by NSF awards CCF-2327010 and CCF-2440113.
        }
	\and Sahil Singla\thanks{
        (ssingla@gatech.edu)
        School of Computer Science,
        Georgia Tech.
        Supported in part by NSF awards CCF-2327010 and CCF-2440113.
        }
}

\maketitle

\begin{abstract}

\noindent Competitive analysis is central to the study of online algorithms, but upper bounds are often highly problem-specific. We develop a more unifying methodology via the minimax viewpoint. Guided by Yao's principle, we reduce worst-case competitive analysis to Bayesian online design under an arbitrary correlated prior over arrival sequences. For such a prior, let $X^*$ be the hindsight-optimal fractional solution for the realized instance, and let $X^{(t)}=\mathbb E[X^*\mid \mathcal F_t]$ be its posterior process. Our guiding rule is \emph{posterior matching}: at each time $t$, choose the feasible online action that tracks the current posterior $X^{(t)}$ as closely as the online constraints permit.

\medskip

\noindent We show that this single principle yields optimal or near-optimal guarantees for several classical online fractional problems, including set cover, load balancing, matching and more general resource-allocation problems, recovering or improving state-of-the-art bounds in these settings with norm/concave objectives. Via known rounding reductions, it also yields randomized integral guarantees for weighted paging, MTS on star metrics, and ski-rental. At a technical level, our analysis reduces competitive guarantees to key probabilistic inequalities for the vector martingales generated by the posterior of the offline optimum. The resulting framework gives a reusable route from Bayesian online design under arbitrary correlated priors to information-theoretic worst-case competitive guarantees.

\end{abstract}

\bigskip

\vspace{-0.8cm}
\setcounter{tocdepth}{1} 
{\footnotesize   \tableofcontents}

\newpage

\newcommand{\Pp}{\mathbb{P}}
\newcommand{\one}{\mathbf{1}}

\section{Introduction}
    Online algorithms study sequential decision-making without knowledge of future inputs. Their standard benchmark is \emph{competitive analysis}, which compares the performance of an online algorithm to the offline optimum that sees the entire input sequence in hindsight. Competitive analysis for online algorithms was formalized by Sleator and Tarjan \cite{SleatorTarjan-85}, but was implicitly studied even earlier, for instance in Graham’s work on greedy scheduling \cite{Graham-66}. Over the past four decades, competitive analysis has become a central paradigm and has inspired a rich toolkit, including potential-function and amortized arguments, as well as online primal-dual methods.

Despite these powerful techniques, online algorithms are still often highly tailored: one typically designs a problem-specific algorithm together with a matching potential function, primal-dual pair, or other specialized analytic argument. A broader goal is to identify \emph{unifying} principles for online algorithm design. In this work, we make progress toward this goal for online \emph{fractional} problems through the minimax viewpoint.


Yao's minimax principle~\cite{Yao-FOCS77} is one of the most powerful tools in the theory of online algorithms. 
The observation is that the worst-case analysis of algorithms can be viewed as a two-player zero-sum game: the algorithm first chooses a randomized strategy, and then the adversary, after seeing that strategy, chooses an input sequence. But the classical von Neumann minimax theorem states that (for finite simultaneous games) the order of the players may be swapped. Inspired by this, Borodin, Linial, and Saks \cite{BLS-JACM92} gave a powerful lower-bound method: to rule out an $\alpha$-competitive randomized algorithm, it suffices to consider an adversary that first chooses a distribution over input sequences (possibly correlated across time) for which then every deterministic online algorithm, after having knowledge of this distribution, performs worse than $\alpha$. This is now the standard tool for proving impossibility results in online algorithms.

But minimax is an equality, not just a one-sided lower-bound principle. At the existential level, it suggests a different route to upper bounds: to prove the existence of an $\alpha$-competitive randomized algorithm against oblivious adversaries, it is enough to show that for every distribution $\cD$ over input sequences there exists an online algorithm $\Alg^{\cD}$ that knows this input distribution and whose expected performance under $\cD$ is within an $\alpha$-factor of the offline optimum.
In other words, this reduces worst-case competitive analysis to average-case (Bayesian) online algorithm design under an arbitrary but known correlated prior.


The difficulty, of course, is that these priors are completely general. They may encode arbitrary correlations across the entire arrival sequence, with no product structure or independence to exploit. 
This helps explain why minimax has had enormous impact as a hardness tool, but comparatively limited\footnote{See \Cref{sec:related} for notable examples where minimax has helped with some individual problems.} impact as a general recipe for online algorithm design.

\begin{quote}
    \emph{Can we find generic techniques to design and analyze  online algorithms for an \emph{arbitrary prior} over arrival sequences, and thereby obtain worst-case guarantees via minimax?}
\end{quote}

Our answer is affirmative in the \emph{fractional} setting for several canonical online problems using essentially the same natural high-level principle, which we denote by \emph{posterior matching}:

\begin{quote}
\noindent\textbf{Posterior matching principle.}
For a prior $\cD$, let $X^*$ denote the hindsight-optimal solution in the relevant fractional relaxation. At time $t$, after observing the history $\mathcal F_t$, choose the feasible online action that matches the posterior $\mathbb E[X^* \mid \mathcal F_t]$ as closely as the online constraints permit.
\end{quote}

 Applications of this single principle yield optimal or  near-optimal guarantees for online fractional set cover, online fractional load balancing, online fractional matching and more general resource-allocation problems. Moreover, for problems that admit known online rounding reductions, such as weighted paging and MTS on star metrics, our fractional guarantees immediately translate into randomized competitive algorithms in the integral model as well. At a technical level, our analyses reduce competitive guarantees to maximal inequalities for vector martingales generated by the posterior of the offline optimum. Across problems, the algorithmic rule changes very little; what varies is the feasibility correction and the martingale inequality. This brings us closer to a more unified understanding of competitive analysis for online algorithms.

To make this viewpoint concrete, we begin with a simple warm-up for unweighted set cover under the promise $\mathrm{OPT}=1$. This example already captures the two recurring ingredients in the paper: (i)~posterior matching determining the algorithm, and (ii)~a martingale maximal inequality controlling its cost.

\subsection{An Illustrative Example: Unweighted Set Cover} \label{sec:introExample}

In online fractional unweighted set cover \cite{AlonAA-STOC03}, there are $m$ sets $S_1,\ldots,S_m$, and elements arrive online. When element $e_t$ arrives, the algorithm observes which sets $S_i$ contain it. Using the available information, at every time $t$ the algorithm needs to maintain a vector $x^{(t)} \in [0,1]^m$ (selecting the amount of each set) that is non-decreasing over time and fractionally covers the arriving element, namely 
$\sum_{i : e_t \in S_i} x_i^{(t)} \ge 1$ for every $t$.
The objective is to minimize $\sum_{i=1}^m x_i^{(T)}$, i.e., the total amount of sets selected in the end.

By the minimax discussion above, it suffices to fix an arbitrary prior over arrival sequences and design an algorithm against this prior. To keep the example especially transparent, assume that the offline optimum has value $1$ on every arrival sequence. Then the optimal fractional solution $X^*$ for the realized sequence lies in the unit simplex, and hence the posterior process
$X^{(t)} := \mathbb E[X^* \mid \mathcal F_t]$
is a simplex-valued martingale.

Since we are not allowed to decrease $x^{(t)}$, the posterior matching rule now suggests the following algorithm: at time $t$, maintain the coordinatewise running maximum of the posterior,
\[
x_i^{(t)} := \max_{0\le s\le t} X_i^{(s)}
\qquad \forall i\in[m].
\]
This is feasible because for the element arriving at time $t$,
\[
\sum_{i : e_t \in S_i} x_i^{(t)}
\;\ge\;
\sum_{i : e_t \in S_i} X_i^{(t)}
\;=\;
\mathbb E\!\Big[\sum_{i : e_t \in S_i} X_i^\star \, \mid\, \mathcal F_t\Big]
\;\ge\; 1,
\]
since $X^*$ covers $e_t$ in every scenario.
Thus the total cost $\mathrm{ALG}$ of the algorithm
is $\sum_{i=1}^m \max_{0\le t\le T} X_i^{(t)}$. We will show that $\E[\mathrm{ALG}] \le O(\log m)$, which matches the bound of \cite{AlonAA-STOC03}.

Let $\mu_i := \mathbb E[X_i^\star] = X_i^{(0)}$. For each fixed $i$, the scalar process $(X_i^{(t)})_{t=0}^T$ is a $[0,1]$-valued martingale with starting value $\mu_i$. By Doob's maximal inequality (a maximal analog of Markov's inequality), 
\[
\Pr\Big[\max_{0\le t\le T} X_i^{(t)} \ge \lambda\Big]
~\le~
\frac{\mu_i}{\lambda}
\qquad \forall \lambda \in [\mu_i,1].
\]
Integrating this tail bound gives Doob's $L\log L$ martingale inequality, bounding the maximum by the initial value $\mu_i$:
\[
\mathbb E\Big[\max_{0\le t\le T} X_i^{(t)}\Big]
\,=\,
\mu_i + \int_{\mu_i}^1 \Pr\Big[\max_{0\le t\le T} X_i^{(t)} \ge \lambda\Big]\,d\lambda
~\le~
\mu_i + \int_{\mu_i}^1 \frac{\mu_i}{\lambda}\,d\lambda
~=~
\mu_i + \mu_i \log \frac{1}{\mu_i}.
\]
Therefore,
\[
\mathbb E[\mathrm{ALG}]
\,=\,
\sum_{i=1}^m \E\Big[\max_{0\le t\le T} X_i^{(t)}\Big]
\,\le~
\sum_{i=1}^m \mu_i
+
\sum_{i=1}^m \mu_i \log \frac{1}{\mu_i}
~\le~ 1+\log m,
\]
where we used $\sum_i \mu_i = \mathbb E[\sum_i X_i^\star]= \E[\OPT] = 1$ and the entropy bound
$\sum_i \mu_i \log(1/\mu_i)\le \log m$.

\medskip
\Cref{sec:setCover} removes the simplifying promise $\mathrm{OPT}=1$ and proves the full  set cover guarantee, even with  norm objectives.

\subsection{Our Contributions and Technical Overview}

Our main conceptual contribution is to use the minimax viewpoint as a route toward a more unified understanding of online algorithm design. The standard minimax reduction lets us focus on arbitrary correlated priors over arrival sequences; our contribution is to show that, in this Bayesian model, essentially the same posterior matching principle yields optimal or near-optimal guarantees across several canonical online fractional problems.

At a technical level, the main contribution is a collection of vector-valued martingale inequalities under general norms that are crucial for analyzing the posterior matching process. This leads, for example, to the notion of $C$-good norms, which can be viewed as a vector-valued analogue of the scalar Doob inequalities used in Section~1.1.

This overview has three parts. We first describe the posterior matching framework under arbitrary correlated priors and record the standard minimax reduction that transfers such guarantees to worst-case randomized guarantees. We then summarize the resulting applications. Finally, we explain the two recurring martingale patterns---coordinatewise maxima and diagonal sums---that serve as the common analytic engine.

\subsubsection{Framework: Correlated Priors and Posterior Matching}

We work in the correlated stochastic online model, where an input instance $I$ is drawn from a known finitely supported distribution $\mathcal{D}$ over arrival sequences and then revealed online to the algorithm. For a minimization problem, an algorithm $\Alg_{\cD}$ is $(\alpha,\beta)$-competitive in expectation with respect to $\mathcal{D}$ if 
\[
\E[\Alg_{\cD}(I)] \le \alpha\,\E[\OPT(I)] + \beta,
\]
with expectation over the random instance $I \sim \mathcal{D}$ and the internal randomness of the algorithm. For maximization problems, the inequality is reversed in the usual way.

Given such a prior $\mathcal{D}$, let $X^*$ denote the hindsight-optimal fractional solution for the realized instance, and let
\[
X^{(t)} := \E[X^* \mid \mathcal F_t]
\]
be its posterior process. This vector-valued martingale is the central state variable in our framework. The online algorithm is then obtained by following the current posterior, after applying only the mildest correction required to ensure feasibility. This is exactly the posterior matching principle introduced above.

This Bayesian viewpoint is justified by the following standard reduction, which lets us pass from guarantees under arbitrary correlated priors to worst-case randomized guarantees.

\begin{proposition}[Standard minimax reduction, informal]
If every finitely supported correlated prior over arrival sequences admits an $(\alpha,\beta)$-competitive online algorithm in expectation, then  there exists a randomized $(\alpha,\beta)$-competitive online algorithm against oblivious adversaries.
\end{proposition}

While several minimax results of this flavor are known, we did not find a complete proof specialized to online algorithms in exactly this form, so Appendix~\ref{sec:minimax} records a formal statement and proof using standard ingredients: 
finite-dimensional minimax together with a compactness argument. Thus, the substantive part of the paper lies in proving that posterior matching yields  competitive guarantees in the correlated stochastic online model, which we discuss in the subsequent part.

\subsubsection{Applications}

Table~\ref{tab:intro-applications} summarizes the main applications of our framework. In each case, we first prove the guarantee against an arbitrary correlated prior and then invoke the minimax reduction. While the posterior process
$X^{(t)} := \E[X^* \mid \mathcal F_t]$
is common throughout, the analyses fall into two recurring martingale inequalities, which we discuss in \Cref{sec:introMartingale}. 

\begin{table*}[t]
\centering
\small
\renewcommand{\arraystretch}{1.08}
\begin{tabular}{@{}p{0.26\linewidth}p{0.70\linewidth}@{}}
\toprule
\textbf{Problem} & \textbf{Informal guarantee} \\
\midrule
Online fractional set cover
& $O(C\log f)$-competitive for every $C$-good monotone norm. In particular: $O(\log f)$ for the standard weighted objective, for Orlicz norms, and for all Lov\'asz extension norms; and $O(\log m\cdot \log f)$ for general symmetric norms. \\[2mm]

Online fractional matching and resource allocation
& $(1-1/e)$-competitive for bipartite matching, and more generally, Devanur--Jain-type concave returns and concave DR-submodular objectives. \\[2mm]

Online fractional load balancing
& $O(C\log m)$-competitive whenever the dual norm is $C$-good. In particular: single logarithm guarantees for makespan, Top-$k$, and Orlicz norms; and a log-square guarantee for general symmetric norms. \\[2mm]

Weighted paging, MTS on star metrics, and ski-rental
& 
Logarithmic guarantees for weighted paging and star-MTS, and the optimal ski-rental ratio, via posterior matching and known rounding reductions where applicable.\\
\bottomrule
\end{tabular}
\caption{Summary of our main applications. All guarantees are first proved against arbitrary correlated priors; worst-case randomized guarantees then follow from the minimax reduction.}
\label{tab:intro-applications}
\end{table*}

\medskip\textbf{Online fractional set cover.}
In online fractional set cover, elements arrive one by one and the algorithm must maintain a nondecreasing fractional choice of sets that covers every revealed element. The classical logarithmic competitive bounds for the standard weighted objective go back to \cite{AlonAA-STOC03}. More recent work has studied richer objectives of the form $\|x^{(T)}\|$, where one measures a monotone norm of the final set-selection vector \cite{AzarBCCCG0KNNP16,NS-ICALP17,KMS-STOC24}.

In \Cref{sec:setCover}, we show that capped posterior matching is $O(\log f)$-competitive for the standard weighted objective, even under arbitrary correlations, where $f$ is the maximum number of sets containing any arriving element. More generally, \Cref{thm:c-good-set-cover} shows that the same algorithm is $O(C\log f)$-competitive for every ``$C$-good'' monotone norm. This recovers the state of the art for symmetric and Orlicz norms from \cite{KMS-STOC24}. It also sharpens the guarantee for ordered norms: although ordered norms are symmetric and were therefore already covered in \cite{KMS-STOC24}, that route gave a log-square bound, whereas our analysis yields a single-logarithmic guarantee. More broadly, our approach goes beyond prior symmetric-norm frameworks by giving the first guarantees for general non-symmetric Lov\'asz extension norms, a class that also includes weighted matroid rank norms \cite{PRS-APPROX23}. Analytically, the proof is driven by a max-martingale inequality for the coordinatewise running maximum of the posterior process; see \Cref{sec:introMartingale}.

\medskip\textbf{Online fractional matching and resource allocation.}
In online fractional matching and resource allocation, items arrive over time and must be assigned irrevocably, possibly fractionally, to offline agents. This includes classical online bipartite matching \cite{KVV-STOC90}, Adwords and generalized matching \cite{MSVV-Journal07}, concave-return model of Devanur and Jain \cite{DJ-STOC12}, and recent DR-submodular framework of \cite{Patton-SODA26}. It also captures important special cases such as free-disposal ad allocation, whole-page optimization, and online submodular assignment \cite{FKMMP-WINE09,DHKMY-TEAC16,HJPSZ-FOCS24}.

In \Cref{sec:recAlloc}, we show that posterior matching is \((1-1/e)\)-competitive for online fractional bipartite matching in an equivalent capped-value formulation. We then extend the same viewpoint to Devanur--Jain-type settings with concave returns (\Cref{thm:concaveMatching}) and to the more general concave DR-submodular setting (\Cref{thm:ocdra}), again obtaining the optimal $(1-1/e)$ factor. The key object here is not a running maximum but a diagonal martingale process, and the proof reduces to a martingale inequality that lower-bounds the corresponding diagonal sum; see \Cref{sec:introMartingale}.

\medskip\textbf{Online fractional load balancing.}
In online fractional load balancing, jobs arrive with machine-dependent processing vectors and the algorithm must fractionally assign each job across $m$ machines so as to minimize a norm of the final load vector. This generalizes classical makespan minimization \cite{AzarNR95,AspnesAFPW-JACM97} and connects to the recent line of norm-sensitive online algorithms developed in \cite{KMPS-FOCS25,KMS-STOC24,KMS-SODA23}. In \Cref{sec:loadBalancing}, \probmatch{} takes an especially clean form: the posterior assignment of the current job is already feasible, so no correction is needed.

We show that this gives a single-logarithmic competitive ratio for makespan, and more generally an $O(C\log m)$ guarantee whenever the \emph{dual norm} is $C$-good (\Cref{thm:good-dual-online}). In particular, we obtain single-logarithmic bounds for Top-$k$ and Orlicz norms, and a log-square guarantee for general symmetric norms, matching the best-known bounds from \cite{KMS-STOC24} (for the harder integral problem) via a very different route. Here the proof again goes through a max-martingale inequality, but only after a duality reduction: one chooses a dual witness for the final load vector, tracks its posterior martingale, and controls its coordinatewise maximum; see \Cref{sec:introMartingale}.

\medskip\textbf{Weighted paging, MTS on star metrics, and ski-rental.}
In \Cref{sec:pagingMTS}, we turn to state-based online problems with natural fractional formulations, including weighted paging, metrical task systems, and ski-rental. These are among the most classical benchmarks in competitive analysis; see for example \cite{SleatorTarjan-85,BLS-JACM92,KMMO-SODA90,BBN-JACM12,BCLLM-STOC18,BGMN-SODA19}. Our contribution here is not an explicit algorithmic improvement over the best specialized algorithms, but rather to show that the same posterior matching principle also recovers the classical guarantees in these settings.

More precisely, we show optimal or near-optimal competitive ratios in the fractional model, which is known to immediately imply results for the integral setting via simple reductions. For weighted paging and MTS on star metrics, known (near-)lossless rounding reductions then transfer these guarantees to randomized integral algorithms with the same asymptotic guarantees. For ski-rental, the posterior matching view already gives a direct randomized algorithm. Hence, our results in this section demonstrate some applications for which \probmatch{} techniques imply results for integral online algorithms, rather than just fractional. Our proofs here are closest in spirit to our set-cover analysis: one tracks posterior processes for pages or states and then applies max-martingale bounds.


\subsubsection{Martingale inequalities as the analytic engine} \label{sec:introMartingale}
The common analytic theme in our results is that posterior matching reduces competitive analysis to controlling one of two martingale functionals: a coordinatewise running maximum or a diagonal process. For set cover, load balancing, and the state-based problems of \Cref{sec:pagingMTS}, the key object is the \emph{coordinatewise running maximum} of a posterior martingale. For matching and resource allocation, the key object is instead its \emph{diagonal}, revealed one time step at a time.

\paragraph{Bounding the Expected Running Maximum.} In the warm-up \Cref{sec:introExample}, we used Doob's $L \log L$ inequality: for a $[0,1]$-valued martingale with mean $\mu$, the expected running maximum is at most $\mu + \mu \log(1/\mu)$. Our results for online set cover, online load balancing, and the state-based problems require vector analogues of exactly this phenomenon.
The relevant object is the coordinatewise maximum $M$ of the posterior matching vector martingale, $X^{(t)} = \E[X^* \mid \cF_t]$, whose coordinates are
\[
M_i := \max_{0\le t\le T} X_i^{(t)}.
\]
We say that a norm $\|\cdot\|$ is \emph{$C$-good} (\Cref{def:c-good}) if for any martingale sequence \((X\up{t})_{t\in[T]}\) over $[0,1]^m$  with $\|X\up{T}\| \leq B$ and any 
$\varepsilon \in (0,1]$, it satisfies the following max-martingale inequality:
\[
\mathbb E\big[\|(M-\varepsilon \mathbf 1)_+\|\big]
    \;\le\;
C \log(1/\varepsilon)\cdot  B.
\]
This is the vector analogue of the scalar inequality from \Cref{sec:introExample}. This bound almost immediately implies the set-cover guarantee. In \Cref{sec:setCover}, we show that many natural norm classes satisfy this property, thereby turning one martingale inequality into a collection of competitive guarantees.

\begin{theorem}[Informal \Cref{lemma:Cgood}]
All Lov\'asz extension norms are $1$-good, all Orlicz norms are $2$-good, and all monotone symmetric norms are $O(\log m)$-good.
\end{theorem}

Perhaps more surprisingly, the same max-martingale inequality reappears in load balancing after a duality reduction. There we choose a dual witness $V$ for the final load vector $L$, write $\|L\|=\langle L,V\rangle$, track the posterior martingale $V^{(t)}=\E[V\mid \cF_t]$, and bound the algorithm by the coordinatewise maximum of $V^{(t)}$. Thus $C$-goodness  controls both covering and load-balancing objectives.

\paragraph{Bounding the Diagonal Process.} For matching and resource allocation, the correct object is different.  Given a vector martingale $X^{(1)},\ldots,X^{(T)}\in \R_+^T$, instead of the coordinate-wise maximum, the key object is its \emph{diagonal}:
\[
D_t := X_t^{(t)} \qquad (t\in[T]).
\]
We seek lower bounds on functions of $(D_1,\ldots,D_T)$ in terms of the terminal vector $X^{(T)}$.

To understand the idea, consider the posterior matching principle for online fractional bipartite matching and fix one offline vertex. Let $X^*_t$ denote the fractional amount by which the $t$-th online vertex in the sequence gets matched to this fixed offline vertex in the hindsight optimum. Posterior matching algorithm will match $D_t = \E[X^*_t \mid \cF_t]$ of the $t$-th online vertex instead. 
Our key diagonal martingale lemma (\Cref{lem:diag-martingale-matching}) states that
    \[
    \E\bigg[\min\bigg\{1,~\sum_{t \in [T]}D_t\bigg\}\bigg] ~\geq~ (1 - 1/e) \cdot \E\bigg[\min\bigg\{1,~\sum_{t \in [T]} X\up{T}_t\bigg\}\bigg],
    \]
which implies that the algorithm is $1 - 1/e$-competitive, as the left-hand side is precisely the expected contribution of this fixed offline vertex to the online solution, while the right-hand side is its contribution to the offline optimum. On both sides, the contribution is capped at $1$ to reflect the matching constraint.
In \Cref{sec:recAlloc}, we extend this diagonal viewpoint to concave returns and more general DR-submodular objectives by controlling more general functions of the diagonal process.

Additionally, this diagonal vector appears in our load balancing analysis as well, where we seek to \textit{upper-bound} (rather than lower-bound) the expectation of norms taken over this diagonal process. As mentioned previously, we reveal a close connection between such bounds and $C$-goodness via duality.
Together, coordinatewise maxima and diagonals form our common analytic engine.


\subsection{Beyond Fractional Relaxations and Comparison to Online Primal-Dual}

The main contribution of this paper is a unifying, information-theoretic framework for \emph{fractional} online problems. This emphasis on fractional relaxations is essential: posterior matching naturally proposes the average action $\E[X^* \mid {\mathcal F_t}]$, which need not correspond to a valid integral decision. For some problems, however, known lossless or near-lossless rounding reductions bridge this gap. In particular, for weighted paging and MTS on star metrics, our fractional guarantees immediately imply randomized integral algorithms with the same competitive ratio.

Extending posterior matching more broadly beyond fractional relaxations remains an important direction. We discuss one such extension in  \Cref{sec:integLoadBalancing} for integral load balancing: instead of playing the posterior average, the algorithm samples an integral action from the posterior of the offline optimum. In general this sampling view can be much weaker than averaging, but for load balancing it still recovers the guarantees of \cite{KMS-STOC24} for $p$-supermodular norms. 
This presents some evidence that posterior matching can continue to be useful beyond fractional solutions.

A second limitation is computational: the minimax-plus-posterior framework is information-theoretic and only implies the \emph{existence} of a desired online algorithm, but does not by itself yield implementable algorithms.
It is instructive to compare our viewpoint with the \emph{online primal-dual} method~\cite{BuchbinderNaor-Book09}. This is arguably the most general method currently available for designing (fractional) online algorithms that has been successfully applied to obtain the best-known competitive ratios for a great variety of problems. While online primal-dual presents a general methodology, the exact algorithm and analysis have to be tailored to the specific problem at hand, often involving some tailored (although by this point familiar) update of primal and dual variables. In contrast, our results indicate that in the information-theoretic sense, there may be a single natural principle -- the posterior matching principle -- behind all different applications. Thus, our viewpoint trades off concrete algorithms for unification. 

Extracting concrete algorithms from the posterior matching principle is an exciting direction. We remark that there is precedent for this: in online learning, the work of~\cite{rakhlin2012relax,foster2015adaptive,pmlr-v65-foster17a,foster2018online} does exactly that. Roughly speaking, they use a minimax theorem to reduce the problem to martingale inequality, extract a key function behind proving this inequality (the Burkholder function), and obtain an explicit algorithm that uses this function to guide its decisions.


\subsection{Further Related Work} \label{sec:related}

A minimax viewpoint has been very fruitful in online learning. Starting with the work of \cite{AABR-COLT09}, leading to much follow-up work, e.g.,~\cite{rakhlin2012relax,RV-OR18,BDKP-COLT15,foster2015adaptive,RV-JMLR16,pmlr-v65-foster17a,foster2018online}, posterior-based reasoning and information-theoretic analyses have led to optimal or near-optimal regret guarantees in a variety of online-learning settings; see, e.g., \cite[Chapter 36]{LS-Book20}. Our setting is different in two important respects: the benchmark in competitive analysis is a full offline solution or trajectory rather than a single fixed action, and the guarantee is multiplicative rather than additive. One conceptual contribution of this paper is to show that, despite these differences, posterior matching together with the right martingale inequality can still yield a reusable design principle.

Within competitive analysis, Yao's principle has classically been used as a lower-bound tool. There are important special cases in which minimax ideas have also informed algorithm design. For example, \cite{ChenKLW-STOC99} formulate finite online planning problems as zero-sum games and use minimax to characterize optimal randomized online algorithms. More recently, \cite{BGSZ-ITCS20} use minimax to prove the existence of a good randomized algorithm for the robust/Byzantine secretary problem. Our goal here is different: rather than exploiting minimax in one particular problem, we seek a general design methodology that can be applied across online problems.

For the concrete problems we study, our results recover and unify several previously known guarantees. For online set cover, the classical logarithmic competitive bounds go back to \cite{AlonAA-STOC03}. For online matching and resource allocation, the classical line of work includes the $(1-1/e)$-competitive algorithm of Karp, Vazirani, and Vazirani \cite{KVV-STOC90}, the Adwords/generalized matching framework of Mehta, Saberi, Vazirani, and Vazirani \cite{MSVV-Journal07}, the concave-return model of Devanur and Jain \cite{DJ-STOC12}, and the recent DR-submodular generalization of \cite{Patton-SODA26}. For online optimization beyond $\ell_p$ norms, \cite{KMS-SODA23,KMS-STOC24} developed norm-sensitive online algorithms and approximation frameworks; in the fractional setting, our results recover the guarantees of \cite{KMS-STOC24} for symmetric and Orlicz norms, improve ordered norms from a log-square to a
single-log, and give the first guarantees for
general non-symmetric Lov\'asz extension norms. 
 For weighted paging, MTS on star metrics, and related state-based online problems, see \cite{BLS-JACM92,BBN-JACM12,BCLLM-STOC18,BGMN-SODA19}. Our contribution for these problems is not an explicit algorithmic improvement, but rather to show that the  posterior matching viewpoint also recovers the optimal competitive ratios once combined with the known fractional-to-integral reductions.

\section{Online Set Cover} \label{sec:setCover}

We begin by formally defining the Online Fractional Set Cover problem with monotone norm objective (recall that a norm is monotone if for every pair of vectors $0\le u\le v$ coordinatewise we have $\|u\|\le \|v\|$). We note that the standard online fractional set cover is obtained by using the weighted $\ell_1$ norm $\|x\| = \sum_i c_i x_i$.  

\begin{definition}[Online Fractional Set Cover with norms]
Fix $m\in\mathbb{N}$ sets $S_1,\dots,S_m$.
Elements arrive online in rounds $t=1,\dots,T$.
When element $e_t$ arrives, the algorithm observes its membership in the $m$ sets.
The algorithm must maintain a nondecreasing vector $x^{(t)}\in[0,1]^m$
(i.e., $x^{(t)}\ge x^{(t-1)}$ coordinatewise) that satisfies the covering constraints
\[
\sum_{i : e_t \in S_i} x^{(t)}_i \;\ge\; 1 \qquad\text{for all } t=1,\dots,T.
\]
Given a norm $\|\cdot\|$, the algorithm seeks to minimize the final cost $\|x\up T\|$.
\end{definition}

Using the minimax approach, we focus on the correlated stochastic online model, where an instance is sampled from a known distribution $\cD$, with correlated arrivals, and presented online to the algorithm. 

\medskip\textbf{Capped \ProbMatch{}.} Let $X^* \in [0,1]^m$ be the (random) offline fractional optimal solution of an instance distributed according to  $\cD$. Following the standard \probmatch{} paradigm, we want our algorithm to have the set selection $x\up t$ at time $t$ match the expected hindsight optimum given our current knowledge, i.e. $\E[X^* \mid \cF_t]$. However, since we can only increase the vector $x$, we will instead maintain $x$ to be the maximum of this value over the history up to time $t$, i.e. $x\up{t}_i \leftarrow  \max_{s \leq t} \E[X^*_i \mid \cF_s]$.

Finally, there is one more adjustment we must make to this ``standard'' \probmatch{} algorithm, which is to ensure that the selections we make are not too small. This ensures that the martingale inequalities we employ give meaningful bounds. We achieve this by ignoring any set selection $x\up{t}_i$ smaller than $\frac{1}{2f}$, where $f  \leq m$ is the maximum number of sets to which an element belongs; to ensure feasibility, we double all remaining selections. Formally, our final algorithm is given in \Cref{alg:capped-pm-set-cover}.

Before analyzing its cost, we quickly check that \Cref{alg:capped-pm-set-cover} indeed produces a feasible solution: for every element $e_t$, the  feasibility of $X^*$ guarantees the coverage $\sum_{i : e_t \in S_i} X^*_i \ge 1$ in every scenario, thus our algorithm's selection satisfies
\begin{align} \label{eq:setCoverFeasible}
\sum_{i : e_t \in S_i} x\up{t}_i ~\geq~ 2\sum_{i : e_t \in S_i}\left(\E\left[X^*_i \mid \cF_t \right] - \frac{1}{2f}\right) ~\geq~ 2\E\left[\tsum_{i : e_t \in S_i} X^*_i \mid \cF_t \right] - 1~\geq~ 2-1 ~=~ 1,
\end{align}
and so we also cover this element. 

\begin{algorithm}
\caption{Capped \ProbMatch{} for Set Cover}\label{alg:capped-pm-set-cover}
Initialize $x\up{0}_i = 0$ for all $i$.

Upon each arrival $t$, for each set $i$ do the following:
\begin{enumerate}
    \item Compute $X\up{t}_i := \E[X^*_i \mid \cF_t]$.
    \item Set $x\up{t}_i = \max\left\{x\up{t-1}_i,~ 2(X\up{t}_i - \frac{1}{2f})^+\right\}$.
\end{enumerate}
\end{algorithm}


\subsection{Warm-up: Weighted Set Cover}\label{sec:set-cover-warmup}

As a warm-up for the general norm guarantees, we start with weighted Set Cover, namely set $i$ has a non-negative cost $c_i$ and the total cost of the algorithm is given by $\sum_i c_i x\up{T}_i$ (which is a weighted $\ell_1$ norm).

\begin{theorem}\label{thm:set-cover-warmup}
    The capped \probmatch~algorithm (\Cref{alg:capped-pm-set-cover}) is $O(\log f)$-competitive for Online Fractional Weighted Set Cover in the correlated stochastic online model.
\end{theorem}

Given the feasibility of \Cref{alg:capped-pm-set-cover} from \eqref{eq:setCoverFeasible}, it suffices to bound its cost, which we do using the following consequence of Doob's martingale inequality.

\begin{lemma}\label{lem:max-martingale}
    Let $X_1, \dots, X_T \in [0,1]$ be a non-negative  martingale. Then for any $\e \in (0,1]$,
    \[
    \E\left(\max_{t \in [T]} X_t - \e\right)^+ \leq~ \E[X_T] \cdot \log(1/\e).
    \]
\end{lemma}
\begin{proof}
    By Doob's martingale inequality~\cite[Theorem 5.4.2]{durrett2010probability}, we have $\Pr\Big[\max_{t \in [T]} X_t \geq u\Big] \leq \frac{\E[X_T]}{u}$ for each $u > 0$. Thus,
    \[
    \E\left(\max_{t \in [T]} X_t - \e\right)^+ =~~ \int_\e^1 \Pr\bigg[\max_{t \in [T]} X_t \geq u\bigg] \text{d}u ~\leq~ \int_\e^1 \frac{\E[X_T]}{u} \,\text{d}u ~=~ \E[X_T] \cdot \log (1/\e).   \qedhere
    \]
\end{proof}

Then using the algorithm's set selection rule $x\up{T}_i = 2\,\big(\max_{t \in [T]} X\up{t}_i - \frac{1}{2f}\big)^+$, the previous lemma, and the observation that the final posterior-matching vector $X\up{T} = \E[X^* \mid \cF_T]$ (i.e., all information revealed) equals the optimum $X^*$ itself, we can upper bound the total cost of the algorithm as 
\[
\E[\Alg] = \sum_{i\in[m]}c_i\, \E[x_i\up{T}] ~\leq~ 2\sum_{i \in [m]}c_i\,\E[X\up{T}_i]\cdot \log(2f) = 2\sum_{i \in [m]}c_i\,\E[X^*_i]\cdot \log(2f) = 2\log(2f)  \E[\Opt],
\]
where the inequality uses \Cref{lem:max-martingale}.
This proves the guarantee of \Cref{thm:set-cover-warmup}.


\subsection{Extending to General Norms via $C$-Goodness}

Inspecting the previous proof we see that the only step that used the structure of the weighted-cost objective was around the application of the maximal inequality of \Cref{lem:max-martingale}; all the other steps only used the definition of the algorithm and $X\up{T} = X^*$. Thus, to generalize the previous results to other norm objectives, we then consider those that satisfy precisely that maximal inequality. This motivates the following definition.

\begin{definition}[\(C\)-good norms]\label{def:c-good}
A monotone norm \(\|\cdot\|\) on \(\mathbb{R}_+^m\) is \emph{\(C\)-good}  if the following holds:  for every \([0,1]^m\)-valued martingale  $(X^{(t)})_{t=0}^T$  where the final vector is bounded $\|X\up{T}\| \leq B$ almost surely for some deterministic $B\geq 0$, we have
\[
\mathbb{E}\|(M - \e\mathbf{1})^+\|
\;\le\;
C
\log(1/\e) \cdot B, \qquad \forall \e \in (0,1],
\]
where $M := \big(\max_{0 \le t \le T} X\up{t}_i\big)_{i \in [m]}$.
\end{definition}

\paragraph{Bounded-$\Opt$ Assumption.} To get a competitive bound using \Cref{def:c-good}, we must also assume that we are given a deterministic bound $\widehat{\Opt}$ such that $\Opt$ is always bounded in the range $\Opt \in [\frac{1}{2} \cdot \widehat{\OPT},~\widehat{\Opt}]$. We call this the \emph{bounded-$\Opt$ assumption.} In the standard adversarial model of online set cover, this assumption can be made without loss of generality (paying a constant in the competitive ratio) using the usual ``guess-and-double'' trick: Start with a guess $\widehat{\Opt}$, and if it is incorrect, double the guess and restart the algorithm.

By applying the minimax principle to the adversarial model \textit{after} reducing the problem using the guess-and-double trick, it suffices to examine the correlated stochastic model under the bounded-$\Opt$ assumption. By minimax, competitive bounds for this setting imply the existence of equally competitive algorithms for the adversarial model with the bounded-$\Opt$ assumption. In turn, this yields competitive bounds for the full adversarial setting (without bounded-$\Opt$ assumption) by the guess-and-double trick.

Under this assumption, we see that when the objective function for online set cover is a $C$-good norm, we can bound the competitive ratio of \Cref{alg:capped-pm-set-cover}. 

\begin{theorem}\label{thm:c-good-set-cover}
    Consider the Online Fractional Set Cover with a norm objective $\|\cdot\|$ that is $C$-good.  Then, \Cref{alg:capped-pm-set-cover} is  an $O(C\log f)$-competitive algorithm for this problem in the correlated stochastic online model with the bounded-$\Opt$ assumption.
\end{theorem}

\begin{proof}

Let $M := (\max_{t\in[T]}X\up{t}_i)_{i \in m}$, so the cost of the algorithm is given by $\Alg = 2\big\|\big(M - \frac{1}{2f} \cdot \mathbf{1}\big)^+\big\|.$ Recall $X\up{T} = X^*$ and by assumption $\|X\up{T}\| \le \widehat{\OPT}$. 
Thus, applying the $C$-goodness inequality from \Cref{def:c-good} with $\varepsilon=1/(2f)$ gives 
\[
\E[\Alg] ~\leq~ 2C\log(2f) \cdot \widehat{\OPT} ~\leq~ 4C\log(2f) \cdot \Opt,
\]
thus proving the theorem. 
\end{proof}

\paragraph{Examples of $C$-good norms.} Crucially, we show that broad classes of norms are indeed $C$-good. We first recall the definition of these norms, and then state the goodness result for all of them. We start with (weighted) Orlicz norms.

\begin{definition}[Weighted Orlicz norm]
    Given a convex, lower semicontinuous function $\Psi : \Rp \to \Rp$ satisfying $\Psi(0) = 0$ and weights $w_1,\ldots,w_m > 0$, the \emph{weighted Orlicz norm} $\|\cdot\|_{\Psi,w}$ on $\Rp^m$ is defined as
    \[
    \|x\|_{\Psi,w} := \inf\bigg\{\lambda \geq 0 : \sum_{i \in [m]} w_i \cdot\Psi(x_i / \lambda) \leq 1\bigg\}.
    \]
\end{definition}

These norms are fundamental objects in functional analysis (e.g., see book \cite{harjulehto2019generalized}) and have also found use in statistics and computer science; see for example~\cite{orliczRegression1,orliczRegression2} for their application in regression. In particular for $p \ge 1$ and $c \in \Rp^m$, a weighted  $\ell_p$ norm $\|x\|_{p,c} := (\sum_i c_i x_i^p)^{1/p}$ is the weighted Orlicz norm given by the function $\Psi(z) = z^p$ and weights $w_i = c_i$. 

The next class is that of Lov\'asz extension norms. Recall that for a monotone submodular function $f : 2^{[m]} \to \Rp$, the \emph{Lov\'asz extension} $f^L : [0,1]^m \to \Rp$ of $f$ for $x\in [0,1]^m$ is defined as $f^L(x) := \int_0^1 f\big(\{i \in [m] : x_i \geq \lambda\}\big)\d \lambda$. This definition can be naturally extended to domain $\Rp^m$, where it induces a norm which we call the \emph{Lov\'asz extension norm} $L_f(x)$.

\begin{definition}[Lov\'asz extension norm]\label{def:lovasz-extension}
    Given a monotone submodular function $f : 2^{[m]} \to \Rp$ satisfying $f(\emptyset)=0$ and $f(S)>0$ for every nonempty $S$,
    the \emph{Lov\'asz extension norm}  is 
    \[
    L_f(x) := \int_0^\infty f\big(\{i \in [m] : x_i \geq \lambda\}\big)\d \lambda.
    \]
\end{definition}

It is easy to check that $L_f$ is a monotone norm (for instance, see \cite{PRS-APPROX23}).  These capture, for example,  Top-$k$ norms (when $f(S) =  \min\{|S|,k\}$), ordered norms (when $f(S) = \sum_{i=1}^{|S|} w_i$ for non-increasing weights $w_1 \ge w_2 \ge \ldots \ge w_m$), and weighted matroid rank norms (when $f$ is a weighted matroid rank function).

We now establish the $C$-goodness of these norms plus that of general monotone symmetric norms (i.e., monotone norms that are invariant to the permutation of the coordinates).

\begin{lemma} \label{lemma:Cgood}
    \setlength{\parskip}{4pt}
    We have the following $C$-goodness properties:
    \begin{itemize}
        \setlength{\itemsep}{0pt}
        \setlength{\parskip}{2pt}
        \setlength{\parsep}{0pt}
        \item Weighted Orlicz norms are 2-good.
        \item Lov\'asz extension norms are 1-good
        \item Monotone symmetric norms are $2(\log m + 1)$-good.
    \end{itemize}
\end{lemma}

We defer the proof of this lemma to \Cref{sec:proofCGood}. In general, the idea for proving $C$-goodness of each norm will be to show the following extension of Doob's inequality: If $S_\delta := \{i \in [m] : M_i \geq \delta\}$ is the set of coordinates for which $M_i$ exceeds some threshold $\delta \geq 0$, then we seek to show
\[
\E\|\mathbf{1}_{S_\delta}\| \leq C\cdot\frac{B}{\delta}.
\]
From here, we can use the triangle inequality to get $\E\|(M - \e \mathbf{1})^+\| \leq \int_{\e}^1 \E\|\mathbf{1}_{S_\delta}\| d\delta \leq CB \int_{\e}^1 \frac{1}{\delta}d\delta = C\log(1/\e)B$.

Together with Theorem~\ref{thm:c-good-set-cover}, these imply the corresponding competitive ratio for Online Fractional Set Cover under these norms. 

\begin{corollary}\label{cor:Cgood}
     Consider the Online Fractional Set Cover with a norm objective $\|\cdot\|$ under the correlated stochastic online model. If the norm is either weighted Orlicz or Lov\'asz extension, then there is an $O(\log f)$-competitive algorithm. If  the norm is monotone symmetric, then there is an $O(\log f \log m)$-competitive algorithm. 
\end{corollary}

\section{Online Matching and Resource Allocation} \label{sec:recAlloc}

In this section, we study online fractional matching and resource allocation settings, in which $T$ items arrive online and must be fractionally assigned among $m$ offline agents. In such problems, we seek to maximize some monotone objective of the allocations, e.g., the size of the fractional matching or the sum of valuations of the agents. We will study the natural algorithm suggested by the posterior matching principle, which we analyze using a diagonal martingale process.

To introduce our methods, we will start with a warm-up in the setting of Online Fractional Bipartite Matching (\Cref{sec:matching-warm-up}). Next, we extend these ideas to the more general setting of Devanur and Jain \cite{DJ-STOC12}, which captures Matching with Concave Returns (\Cref{sec:devanur-jain-matching}). Finally, we will give our most general results in the setting of \cite{Patton-SODA26}, which captures arbitrary concave, DR-submodular objectives (\Cref{sec:DRsubmod}). In each setting, we will show how the \probmatch{} technique can be used to obtain the optimal $(1-\frac{1}{e})$-competitive ratio in the correlated stochastic model.


\subsection{Warm-up: Online Matching}\label{sec:matching-warm-up}

We start with the simplest case, namely that of Online Fractional Bipartite Matching. 

\begin{definition}
In Online Fractional Bipartite Matching, we have a bipartite graph $G = ([m], [T], E)$ initially hidden to the algorithm. At each time $t \in [T]$, the neighborhood $N(t)$ of $t$ is revealed to the algorithm, and the algorithm must set values $x_{it} \in [0,1]$ for each $i \in N(t)$ such that $t$ is fractionally matched at most once, namely $\sum_{i \in N(t)} x_{it} \le 1$. The goal of the algorithm is to maximize the size of the matching, namely  $\sum_i \min\{1, \sum_{t : i \in N(t)} x_{it}\}$, i.e., each $i$ collects at most value 1 from the nodes matched to it. 
\end{definition}

Note that this definition differs slightly from the traditional formulation of the problem since we have no offline constraints (i.e. constraints of the form $\sum_{t \in N(i)} x_{it} \leq 1$), but instead we simply cap the contribution from each offline vertex at $1$. This is mostly a notational convenience, as it will allow us to use the same problem framework later in more general settings.

\paragraph{\Probmatch{}.} As always, we focus on the correlated stochastic online model, where the graph $G$ is drawn from a known distribution $\cD$ and the neighborhood $N(t)$ of node $t$ is revealed at time $t$. We can now naturally define the \probmatch{} algorithm: let $X^* \in [0,1]^{mT}$ be the optimal offline allocation vector for an instance drawn from $\cD$; at time $t$, the algorithm then sets the fractional allocation $x_{it} = X\up{t}_{it} := \E[X^*_{it} \mid \cF_t]$ of item $t$ for each of the agents $i$.

We have the following result.

\begin{theorem} \label{thm:matching}
    The \probmatch{} algorithm is $(1 - \frac{1}{e})$-competitive for Online Fractional Bipartite Matching under the correlated stochastic online model.
\end{theorem}

\paragraph{Diagonal martingale process.} To prove these competitive bounds for the \probmatch{} algorithm, the key challenge lies in analyzing the diagonal martingale process $X\up{1}_{i1}, X\up{2}_{i2}, \dots, X\up{T}_{iT}$ for each offline vertex $i$, as mentioned in \Cref{sec:introMartingale}. Observe that for any given offline vertex $i \in [n]$, the algorithm's collected (random) value equals $\min\{1, \sum_t X\up{t}_{it}\}$. On the other hand, the offline optimum collects $\min\{1, \sum_t X\up{T}_{it}\}$.
Ideally, we would like to lower-bound $\E\min\{1, \sum_t X\up{t}_{it}\}$   in terms of $\sum_t X^*_{it}= \E\min\{1, \sum_t X\up{T}_{it}\}$ (up to a $(1-1/e)$ factor). Although  $\E\sum_t X\up{t}_{it} = \E\sum_t X\up{T}_{it}$ by the martingale property,  it's not obvious how to lower bound the algorithm's expected value as the $\min$ function is concave. 
To get around this difficulty, we prove the following technical lemma.

\begin{lemma}\label{lem:diag-martingale-matching}
    Let $X\up{1},\ldots,X\up{T} \in \Rp^T$ be a vector martingale. Define the diagonal $D_t := X\up{t}_t$ for $t \in [T]$. 
    Then
    \begin{align}
    \E\bigg[\min\bigg\{1,~\sum_{t \in [T]} D_t\bigg\}\bigg] ~\geq~ (1 - 1/e) \cdot \E\bigg[\min\bigg\{1,~\sum_{t \in [T]} X\up{T}_t\bigg\}\bigg] . \label{eq:diagonal} 
    \end{align}
\end{lemma}

Before proving this lemma, we will see how to use it to complete the analysis of \probmatch{}.

\begin{proof}[Proof of~\Cref{thm:matching}]
    Recall that $(X^*_{it})_{i,t}$ is the optimal allocation vector, and \probmatch{} makes the allocation $(X\up{t}_{it})_i = (\E[X^*_{it} \mid \cF_t])_i$ at time $t$. Fix an offline vertex $i$. Applying \Cref{lem:diag-martingale-matching} to the vector martingale $(X\up{1}_{is})_{s \in [T]}, (X\up{2}_{is})_{s \in [T]}, \ldots, (X\up{T}_{is})_{s \in [T]}$, we get that the value the algorithm gets from this offline node is
    \begin{align*}
        \E\Big[\min\Big\{1,~\sum_{t \in [T]} X\up{t}_{it}\Big\}\Big] \,\geq\, 
        (1 - 1/e)\, \E\Big[\min\Big\{1,~\sum_{t \in [T]} X\up{T}_{it}\Big\}\Big].
    \end{align*}
    Adding this inequality over all offline nodes $i$ and observing that $X\up{T}_{it} = X^*_{it}$, we obtain $\E[\Alg] \ge (1-1/e) \cdot \E[\OPT]$. This gives the desired competitiveness of the \probmatch{} algorithm. 
\end{proof}

Now we prove the diagonal martingale inequality.

\begin{proof}[Proof of \Cref{lem:diag-martingale-matching}]
    A very useful simplifying step is to assume that  $\sum_t X\up{T}_t \leq 1$ always (i.e., the minimum in the right-hand side is equal to $\sum_t X\up{T}_t$). This is without loss of generality because otherwise we can define the truncated version $\bar{X}\up{T}$ of the vector $X\up{T}$ given by 
    \[
    \bar{X}\up{T}_t := \min\{ X\up{T}_t, 1 - \sum_{t' < t} \bar{X}\up{T}_{t'}\},
    \] and work with the truncated martingale $\bar{X}\up{t} := \E[\bar{X}\up{T} \mid \cF_t]$; this truncated martingale gives the same right-hand side as in \eqref{eq:diagonal} (indeed $\sum_t \bar{X}\up{T}_t \le 1$ always, and when $\sum_t X\up{T}_t \le 1$ we have $\sum_t\bar{X}\up{T}_t  = \sum_t X\up{T}_t$), and the truncated diagonal gives lower value in the left-hand side of \eqref{eq:diagonal} ($\bar{X}\up{t}_t = \E[\bar{X}\up{T}_t \mid \cF_t] \le \E[X\up{T}_t \mid \cF_t] = X\up{t}_t$).

    With this simplifying assumption, now define the prefix sum of the diagonal $A_t := D_1 + \ldots + D_t$ and the ``conditional expected suffix at time $t$'' $R_t := X\up{t}_{t} + \ldots + X\up{t}_T$ (i.e., where we use the conditional expectation $X_{t+i}\up{t} = \E[D_{t+i} \mid \cF_t]$ instead of $D_{t+i}$ itself). Using this language,
    and noticing $\E[R_1] = \E[\sum_t X\up{T}_t]$, by assumption $(1-1/e)\cdot \E[R_1]$ is the RHS of \eqref{eq:diagonal}. Thus, we seek to show that 
    \begin{align}
    \E[\min\{1,A_T\}] \stackrel{\text{want}}{\geq} (1-1/e)\cdot\E[R_1] \label{eq:matchingWant}
    \end{align}
    to prove our lemma. 

    For that, consider the potential $N_t$ for each time $t \in [T+1]$ defined as
    \[
    N_t := R_t \cdot e^{\min\{A_{t-1}, 1\}-1} + (A_{t-1} - 1)^+.
    \]
    The key claim is that this potential forms a supermartingale.
    
    \begin{claim} \label{claim:submartingale}
        $(N_t)_t$ is a supermartingale, namely $\E[N_{t+1} \mid \cF_t] \le N_t$ for all $t \in [T+1]$. In particular, $\E[N_{t+1}] \le \E[N_t]$ for all such $t$. 
    \end{claim}
    
    This claim indeed gives the desired inequality \eqref{eq:matchingWant}: by chaining its conclusion, we obtain that $\E[N_{T+1}] \le \E[N_1]$, which unraveling the notation means $\E (A_T - 1)^+ \le \E[R_1] \cdot e^{-1}$; further noticing $\E[A_T] = \E[\sum_t X\up{t}_t] = \E[\sum_t X\up{T}_t] = \E[R_1]$, a little algebra gives
    \[
    \E[\min\{1,A_T\}] \,=\, \E[ A_T] - \E(A_T - 1)^+ \,=\, \E[R_1] - \E(A_T - 1)^+ \geq\, \E[R_1] - \E[R_1] \cdot e^{-1} \,=\, (1-1/e)\cdot  \E[R_1],
    \]    
    thus giving \eqref{eq:matchingWant}. 
    
    For the remainder of the proof, we show~\Cref{claim:submartingale}. To do this, we will crucially use the fact that \(R_t = X\up{t}_t + \E[R_{t+1} \mid \cF_t]\). We consider three cases based on whether the prefix sum $A_{t-1}$ and/or $A_t$ exceeds 1.

    \paragraph{Case 1: $1 \leq A_{t-1} \leq A_t$.} In this case $N_t = R_t + A_{t-1} - 1$ and $N_{t+1} = R_{t+1} + A_t - 1$. and so
    \[
    N_t = R_t + A_{t-1} -1 = X\up{t}_t + \E[R_{t+1} \mid \cF_t] + A_{t-1} -1 = \E[R_{t+1} \mid \cF_t] + A_t - 1 = \E[N_{t+1} \mid \cF_t],
    \]
    and the supermartingale inequality is proved in this case. 
    
    \paragraph{Case 2: $A_{t-1} \leq A_t \leq 1$.} In this case $N_t = R_t \cdot e^{A_{t-1} - 1}$ and $N_{t+1} = R_{t+1} \cdot e^{A_t - 1}$. Now we use the fact that
    \[
    \E[R_{t+1} \mid \cF_t] ~=~ R_t - X\up{t}_t  ~\leq~ R_t \cdot (1 - X\up{t}_t) ~\leq~ R_t \cdot e^{-X\up{t}_t},
    \]
    where the first inequality follows from $R_t \leq \E[\sum_{s \in [T]}X\up{T}_s \mid \cF_t] \leq 1$ by our assumption. Therefore, we get
    \begin{align*}
        \E[N_{t+1} \mid \cF_t] = \E[R_{t+1} \mid \cF_t] \cdot e^{A_{t}-1} \le R_t \cdot e^{-X\up{t}_t} \cdot e^{A_{t}-1}  = R_t \cdot e^{A_{t-1} - 1} = N_t,
    \end{align*}
    and we are also done in this case.

    \paragraph{Case 3: $A_{t-1} \leq 1 \leq A_t$.} In this case $N_t = R_t \cdot e^{A_{t-1} - 1}$ and $N_{t+1} = R_{t+1} + A_t - 1$. Here, we use a combination of the above ideas. Let $a = 1 - A_{t-1}$ and $b = A_{t} - 1$, so $X\up{t}_t = a + b$. Then we have
    \[
    \E[R_{t+1} \mid \cF_t] ~=~ R_t - a - b ~\leq~ R_t \cdot (1 - a) - b ~\leq~ R_t \cdot e^{-a} -b.
    \]
    Therefore, we get
    \[
     \E[N_{t+1} \mid \cF_t] = \E[R_{t+1} \mid \cF_t] + A_t - 1 \le R_t \cdot e^{-a} - b + A_t - 1 = N_t.
    \]
    This proves the last case of ~\Cref{claim:submartingale}, which concludes the proof of~\Cref{lem:diag-martingale-matching}.
\end{proof}


\subsection{Extension to Concave Returns} \label{sec:devanur-jain-matching}
Next, we will show that ideas for the matching setting can be quickly extended to show a $(1 - \frac{1}{e})$-competitive ratio in the more general setting of online matching with concave returns, which captures the general bound\footnote{\cite{DJ-STOC12} actually proves something a little more precise for this setting: They characterize the best possible competitive ratio $\gamma \geq 1-1/e$ as it depends on the functions $V_i$. Here, we just focus on recovering the baseline competitive ratio $1 - 1/e$.} proven by Devanur and Jain \cite{DJ-STOC12}.

\begin{definition}\label{def:devanur-jain-setting}
    In \emph{Online Fractional Matching with Concave Returns}, we have a set $[m]$ of offline nodes, where each $i \in [m]$ has a monotone concave valuation function $V_i : \Rp \to \Rp$. Online nodes $t \in [T]$ arrive one at a time, and upon each arrival $t$, weights $b_{it} \in \Rp$   are revealed for each $i \in [m]$. The algorithm must immediately set values for the fractional allocation $x_{it} \in [0,1]$ of item $t$ to each agent $i$ subject to $\sum_{i \in [m]} x_{it} \leq 1$, and the goal of the algorithm is to maximize the total value obtained by the offline nodes:
    \[
    \sum_{i \in [m]} V_i\bigg(\sum_{t \in [T]} b_{it} x_{it}\bigg).
    \]
\end{definition}

Notice that if $V_i(y) = \min\{1,~y\}$ and each $b_{it} \in \{0,1\}$ for each $i$ and $t$, then we recover the Online Fractional Matching problem. Again we are interested in the correlated stochastic online model, where the functions $V_i$ and the weights $b_{it}$ are drawn from a known joint distribution $\cD$, and at time $t$ the weights $b_{it}$ of the online node $t$ and all  offline nodes   $i$ are revealed. The following is the extension of the previous result to this more general problem. 

\begin{theorem} \label{thm:concaveMatching}
    The \probmatch{} algorithm is $(1 - \frac{1}{e})$-competitive for Online Fractional Matching with Concave Returns under the correlated stochastic online model.
\end{theorem}

The proof of this result follows verbatim the proof of \Cref{thm:matching} once we have generalized the key diagonal martingale inequality~\Cref{lem:diag-martingale-matching} from the (concave) hinge function $y \mapsto \min\{1, y\}$ of the diagonal to an arbitrary monotone concave function. In the next lemma, we do precisely this.

\begin{lemma}\label{lem:diag-martingale-concave}
    Let $X\up{1},\ldots,X\up{T} \in \Rp^T$ be a vector martingale, and let $V : \Rp \to \Rp$ be a monotone concave function. Define the diagonal $D_t := X\up{t}_t$ for $t \in [T]$. Then
    \[
    \E\Big[V\Big(\sum_{t \in [T]} D_t\Big)\Big] \geq (1 - 1/e) \cdot \E\Big[V\Big(\sum_{t \in [T]} X\up{T}_t\Big)\Big] 
    \]
\end{lemma}
\begin{proof}
    The crucial observation is that any monotone concave function like $V$ can be written as a positive combination of (concave) hinge functions: There is a measure $\mu$ over $\R_+$ such that $V(y) = V(0) + \int_0^\infty \min\{\tau, y\}\, \d \mu(\tau)$ for some weighting $\mu$ over $\Rp$ (in particular, $\mu$ is the distribution with CDF given by $F_\mu(y) = V'(0) - V'(y)$; this gives $\d \mu(\tau) = -V''(\tau)\d \tau$ if $V$ is twice differentiable). Thus, by linearity of expectation, it suffices to show the lemma holds when $V(y) = \min\{\tau, y\}$ for each $\tau > 0$. For such settings, we apply \Cref{lem:diag-martingale-matching} on the rescaled martingale $Y\up{t} := \frac{X\up{t}}{\tau}$ to get
    \begin{align*}
    \E\Big[\min\Big\{\tau,~\sum_{t \in [T]} X\up{t}_t\Big\}\Big] 
    =
    \tau \cdot \E\Big[\min\Big\{1,~\sum_{t \in [T]} Y\up{t}_t\Big\}\Big] &\stackrel{\text{Lem~\ref{lem:diag-martingale-matching}}}{\geq} 
    (1 - 1/e) \cdot \tau \cdot \E\Big[\min\Big\{1,~\sum_{t \in [T]} Y\up{T}_t\Big\}\Big]\\
    &~~~=
    (1 - 1/e) \cdot \E\Big[\min\Big\{\tau,~\sum_{\tau \in [T]} X\up{T}_t\Big\}\Big] \enspace .
    \end{align*}
    This concludes the proof of the lemma. 
\end{proof}


\subsection{Extension to DR-Submodular Functions} \label{sec:DRsubmod}

Finally, we will show that the \probmatch{} algorithm even obtains the optimal $(1-\frac{1}{e})$-competitive ratio for the general Online Fractional Resource Allocation problem where we only require that the objective is concave and DR-submodular. We begin by defining our general problem setting.
\begin{definition}\label{def:online-resource-allocation}
    In Online Fractional Resource Allocation, we have $[m]$ offline agents, and $[T]$ items that arrive online. When each item $t$ arrives, the algorithm must immediately set values $x_{it} \in [0,1]$ for each $i \in [m]$ such that $\sum_{i \in [m]} x_{it} \leq 1$. The goal of the algorithm is to maximize a function $f(x)$ of the whole allocation vector $x$, where $f : \Rp^{mT} \to \Rp$ is a monotone function with $f(\mathbf{0}) = 0$.  
    
    Moreover, the function $f$ is initially unknown to the algorithm, and at each time $t$, the algorithm is only revealed the function restricted to arrivals up to time $t$. Formally, at time $t$ the algorithm has access to $f_{[t]} : \Rp^{mt} \to \Rp$ given by $f_{[t]}((x_{is})_{i \in [m],s \in [t]}) := f((x_{is})_{i \in [m],s \in [t]}, \mathbf{0})$.
\end{definition}

For example, in online bipartite matching, the objective $f$ is $f(x) = \sum_{i} \min\{1, \sum_{t : i \in N(t)} x_{it}\}$, where $N(t)$ is the neighborhood of the online vertex $t \in [T]$ in the bipartite graph. In welfare maximization settings, we have $f(x) = \sum_{i} v_i\big((x_{it})_{t \in [T]}\big)$, where $v_i : \Rp^T \to \Rp$ is the valuation function of agent $i$ over the items. We show that (modulo some technical assumptions), as long as $f$ is concave and DR-submodular, the \probmatch{} algorithm is $(1-1/e)$-competitive.

\begin{theorem}\label{thm:ocdra}
    For the Online Fractional Resource Allocation problem  defined in \Cref{def:online-resource-allocation} under the correlated stochastic online model, if the objective function $f$ is upward-differentiable\footnote{We say $f : \Rp^N \to \Rp$ is \emph{upward-differentiable} if at each point $y \in \Rp^N$, there is an \emph{upward-gradient} $\nabla f(y)$ such that $\lim_{\epsilon \to 0^+} \frac{f(y + \epsilon z) - f(y)}{\epsilon} = \ip{\nabla f(y)}{z}$ for all $z \in \Rp$. The reader loses little by simply assuming that $f$ is differentiable; see \cite{Patton-SODA26} for a formalization of familiar concepts (e.g. the chain rule) to upward-differentiable functions.}, concave, and DR-submodular\footnote{We say an (upward)-differentiable function $f : \Rp^N \to \Rp$ is \emph{DR-submodular} if $\nabla f(y) \leq \nabla f(z)$ coordinate-wise whenever $y \geq z \geq 0$ coordinate-wise.}, then the \probmatch{} algorithm is $(1 - \frac{1}{e})$-competitive.
\end{theorem}

Notably, this is the most general class of fractional resource allocation problems where a $(1-\frac{1}{e})$-competitive algorithm is known to exist, matching the results of \cite{Patton-SODA26}. In particular, it captures matching with concave returns from the previous section \cite{DJ-STOC12}, edge-weighted matching with free disposal \cite{FKMMP-WINE09}, whole-page optimization \cite{DHKMY-TEAC16}, and matching with polymatroid constraints \cite{HJPSZ-FOCS24}.

\subsubsection{$\gamma$-Balanced Functions}

To prove our theorem, we use a diagonal martingale lemma combined with a key fact from \cite{Patton-SODA26}: For every function $f$ satisfying the conditions of \Cref{thm:ocdra}, there exists a function $U$ with a certain ``$\gamma$-balanced'' property with respect to $f$. This notion of $\gamma$-balancedness was initially motivated by the primal-dual approach of \cite{Patton-SODA26}. Here, we will instead use it together with a key martingale inequality (\Cref{lem:diag-martingale-cdr}) to bound the performance of the \probmatch{} algorithm.

We begin by defining $\gamma$-balancedness as follows.
\begin{definition}[$\gamma$-balanced]\label{def:gamma-balanced}
     Let $f : \Rp^N \to \Rp$ be a monotone, upward-differentiable function. Then a monotone, upward-differentiable function $U : \Rp^N \to \Rp$ is said to be \emph{$\gamma$-balanced} with respect to $f$ if for every $y \in \Rp^N$, we have
     \[
     f(y) \geq \gamma \cdot\left(U(y) + \hat f(\nabla U(y))\right),
     \]
     where $\hat f(\alpha) := \sup_{y \in \Rp^{N}} \left\{f(y) - \ip{\alpha}{y}\right\}$.
\end{definition}

Next, we will use the following key fact, which guarantees the existence of such a function $U$.
\begin{fact}[\cite{Patton-SODA26}]\label{fact:key-cdr}
    For any concave $f : \Rp^N \to \Rp$, the function $U : \Rp^N \to \Rp$ given by 
    \begin{equation}\label{eq:aux-fn-U}
    U(y) := \frac{1}{e-1} \int_0^1 e^{\tau} \cdot \tau f\left(\frac{y}{\tau}\right) d\tau
    \end{equation}
    is $(1 - 1/e)$-balanced with respect to $f$.
\end{fact}

\subsubsection{Proof of \Cref{thm:ocdra}}

To prove our main result, we introduce the following key martingale lemma, which will allow us to exploit the auxiliary function $U$.

\begin{lemma}\label{lem:diag-martingale-cdr}
    Let $X\up{1}, \dots, X\up{T}$ be a vector martingale in $\Rp^{mT}$ with filtration $(\cF_t)_{t \in [T]}$. Let $U : \Rp^{mT} \to \Rp$ be a (random) DR-submodular, monotone, upward differentiable function such that for each $t \in [T]$, the restriction $U_{[t]} : \Rp^{mt} \to \Rp$ given by $U_{[t]}((x_{is})_{i \in [m],~s \in [t]}) := U((x_{is})_{i \in [m],~s \in [t]}, \mathbf{0})$ is $\cF_t$-measurable. Let $D = (X\up{t}_{it})_{i \in [m],t\in[T]}$ be the diagonal vector where each entry at $(i,t)$ is taken from $X\up{t}$. Then we have
    \[
    \E U(D) \geq \E \ip{\nabla U(D)}{X\up{T}}.
    \]
\end{lemma}

We defer the proof of this lemma for now, and first we see how \Cref{thm:ocdra} follows from \Cref{lem:diag-martingale-cdr} and \Cref{fact:key-cdr}.
\begin{proof}[Proof of \Cref{thm:ocdra}]
    Suppose we have an instance of Online Fractional Resource Allocation satisfying the conditions of \Cref{thm:ocdra}. Let $U$ be the function defined in \Cref{fact:key-cdr}, which depends on the random objective $f$. 
    
    As always, we define the vector martingale $X\up{t}$ over $\Rp^{mT}$ as usual by $X\up{t}_{is} = \E[X^*_{is} \mid \cF_t]$. We also define $D = (X\up{t}_{it})_{i \in [m],t\in[T]}$ to be the diagonal vector, as defined in \Cref{lem:diag-martingale-cdr}. Then the \probmatch{} algorithm obtains the solution given by $x = D$, while the optimum solution is given by $X^* = X\up{T}$.

    Since the restriction $U_{[t]}$ of $U$ on the first $t$ arrivals is determined only by $f_{[t]}$, we have that $U_{[t]}$ is $\cF_{t}$-measurable. Hence, $U$ and $X$ meet the requirements of \Cref{lem:diag-martingale-cdr}, so we have $\E U(D) \geq \E \ip{\nabla U(D)}{X\up{T}}$. We combine this with \Cref{fact:key-cdr} to obtain
    \begin{align*}
        \E[\Alg] = \E[f(D)]
        &\geq 
        \left(1-1/e\right) \cdot \E\left(U(D) + \hat f(\nabla U(D))\right)  && \text{by $\gamma$-balancedness},\\
        &= 
        \left(1-1/e\right) \cdot \E \left(U(D) + \sup_{x}\left\{f(x) - \ip{\nabla U(D)}{x})\right\}\right)\\
        &\geq 
        \left(1-1/e\right) \cdot \E \left(U(D) + f(X\up{T}) - \ip{\nabla U(D)}{X\up{T}})\right)\\
        &\geq
        \left(1-1/e\right) \cdot \E f(X\up{T}) && \text{by \Cref{lem:diag-martingale-cdr},}\\
        &
        = \left(1-1/e\right) \cdot \E[\Opt]. && \qedhere 
    \end{align*}
\end{proof}

Now, we prove our martingale lemma.
\begin{proof}[Proof of \Cref{lem:diag-martingale-cdr}]
    For each $t \in \{0, \dots, T\}$, define $D\up{\leq t} \in \Rp^{mT}$ by
    \[
    D\up{\leq t}_{is} = \begin{cases}
        D_{is} & s \leq t, \\
        0 & \text{otherwise.}
    \end{cases}
    \]
    Furthermore, define $D\up{t} = D\up{\leq t} - D\up{\leq t-1}$. 
    Since $U\ge0$ and $\hat f(\nabla U(0))\ge f(0)=0$ (take $y=0$ in the supremum), balancedness at $0$ implies
     $0=f(0)\ge\gamma U(0)$, and hence $U(0)=0$.
Thus, we can write $U(D)$ as the telescoping sum
    \begin{align*}
        U(D) ~&=~ \sum_{t \in [T]} \left(U(D\up{\leq t}) - U(D\up{\leq t-1})\right) \\
         &=~ \sum_{t \in [T]}\int_0^1 \ip{\nabla U(D\up{\leq t-1} + \delta D\up{t})}{~D\up{t}}d\delta~\geq~ 
        \sum_{t \in [T]} \ip{\nabla U(D\up{\leq t})}{~D\up{t}},
    \end{align*}
    where the inequality follows from DR-submodularity of $U$. We can expand each of these inner products as 
    \[
    \ip{\nabla U(D\up{\leq t})}{D\up{t}} = \sum_{i \in [m]} \nabla_{it} U(D\up{\leq t}) \cdot X\up{t}_{it} = \sum_{i \in [m]} \nabla_{it} U(D\up{\leq t}) \cdot \E[X\up{T}_{it} \mid \cF_t].
    \]
    Notice that since the $D\up{\leq t}$ and the restriction $U_{[t]}$ are $\cF_t$-measurable, the values $\nabla_{it} U(D\up{\leq t}) = \nabla_{it} U_{[t]}(D\up{\leq t})$ are also $\cF_t$-measurable for each $i$. Thus, we have
    \[
    \sum_{i \in [m]} \nabla_{it} U(D\up{\leq t}) \cdot \E[X\up{T}_{it} \mid \cF_t] 
    = 
    \sum_{i \in [m]} \E\left[\nabla_{it} U(D\up{\leq t}) \cdot X\up{T}_{it} \mid \cF_t\right] 
    \geq 
    \sum_{i \in [m]} \E\left[\nabla_{it} U(D) \cdot X\up{T}_{it} \mid \cF_t\right],
    \]
    where the last inequality follows from the fact that $\nabla  U(D\up{\leq t}) \geq \nabla  U(D)$ coordinate-wise since $U$ is DR-submodular.

    Combining these inequalities and taking expectations, we get
    \begin{align*}
        \E U(D)
        &\geq
        \E\Big[\sum_{t \in [T]} \ip{\nabla U(D\up{\leq t})}{~D\up{t}}\Big] \\
        &\geq \E\Big[\sum_{t \in [T]}\sum_{i \in [m]}\E\left[ \nabla_{it} U(D) \cdot X\up{T}_{it} \mid \cF_t\right]\Big]\\
        &=
        \E\Big[\sum_{t \in [T]}\sum_{i \in [m]} \nabla_{it} U(D) \cdot X\up{T}_{it} \Big]
        ~=~ \E\ip{\nabla U(D)}{X\up{T}}.     
    \end{align*}
    This concludes the proof of~\Cref{lem:diag-martingale-cdr}.
\end{proof}

\begin{remark}\label{thm:gamma-balanced-resource-allocation}
    We note that the proof of \Cref{thm:ocdra} can also be used to show that \probmatch{} is $\gamma$-competitive for $\gamma > 1 - \frac{1}{e}$ for special cases of $f$, so long as one can find a $\gamma$-balanced function $U$ (possibly different from \eqref{eq:aux-fn-U}) satisfying a few necessary regularity conditions. Formally, we have the following general statement:

    Suppose we have an instance of online fractional resource allocation under correlated stochastic online model. Suppose as well that there exists a function $U : \Rp^{mT} \to \Rp$, depending on $f$, with the properties that
    \begin{enumerate}
        \setlength{\itemsep}{0pt}
        \setlength{\parskip}{2pt}
        \setlength{\parsep}{0pt}
        \item $U$ is DR-submodular.
        \item $U$ is $\gamma$-balanced with respect to $f$ for some $\gamma > 0$.
        \item The restriction $U_{[t]}$ to the first $t$ arrivals is determined only by $f_{[t]}$ for each $t \in [T]$.
    \end{enumerate}
    Then the \probmatch{} algorithm is $\gamma$-competitive.
\end{remark}

\section{Online Load Balancing} \label{sec:loadBalancing}
    
\label{sec:online-load-balancing}

We study fractional online load balancing on \(m\) unrelated machines over a horizon of \(T\) jobs. At time \(t\), a job arrives with processing vector
\[
p_t = (p_{t1},\dots,p_{tm}) \in \mathbb{R}_+^m,
\]
where \(p_{ti}\) denotes the load incurred on machine \(i\) by assigning one unit of job \(t\) to that machine. After each job $t$ arrives, the online fractional algorithm must fractionally assign $t$ by choosing an assignment vector
\[
x_t = (x_{t1},\dots,x_{tm}) \in \Delta_m
\qquad\text{where}\qquad
\Delta_m := \bigg\{x \in \Rp^m : \sum_{i=1}^m x_i=1\bigg\},
\]
where $x_{ti}$ denotes the amount of job $t$ assigned to machine $i$. After all jobs have been assigned, the total load on machine $i$ is given by
$L_i := \sum_{t \in [T]} x_{ti} p_{ti}$, and the goal of the algorithm is to minimize a monotone norm of the load vector $\|L\|$, where $L = (L_1, \dots, L_m)$. 

In the standard makespan minimization setting, we minimize the $\ell_\infty$-norm, i.e. the maximum load on any machine. In \Cref{sec:lb-warmup}, we will see as a warm-up how our techniques can be applied to this setting. Later, in \Cref{sec:lb-general-norms}, we will extend our methods to handle more general classes of norms.

\medskip\textbf{Bounded $\Opt$ assumption.}
Like in \Cref{sec:setCover}, we make the simplifying assumption that algorithm is given a fixed bound $\widehat{\Opt} \geq 0$ upfront, such that $\Opt \leq \widehat{\Opt} \leq 2\Opt$ on all possible arrival sequences, where $\Opt$ is the hindsight optimum objective value. Recall that in the adversarial arrival model, the standard ``guess-and-double'' trick allows us to  make this assumption without loss of generality by losing a factor $2$ in our competitive ratio. Hence, by applying this reduction to the adversarial problem before applying the minimax principle, it then suffices for us to consider the correlated stochastic model with the bounded-$\Opt$ assumption.

\medskip\textbf{\ProbMatch{}.}
As always, we will focus on the correlated stochastic model, where the sequence of processing vectors $(p_1, \dots, p_T) \sim \cD$ is sampled from some known distribution $\cD$. Let $X^* \in \Delta_m^{T}$ denote the hindsight optimal assignment of all jobs (as a function of the arrival sequence). We also incorporate the bounded $\Opt$ assumption into our model, i.e. have that $\Opt \leq \widehat{\Opt} \leq 2\Opt$ for every instance in the support of $\cD$. 

Let $X\up{1}, \dots, X\up{T}$ be the natural martingale given by $X\up{t} = \E[X^* \mid \cF_t]$ ($\cF_t$ is generated by $p_1,\ldots,p_t$), noting that each $X\up{t}$ is also an assignment of \textit{all} jobs. Our algorithm will follow standard \probmatch{}, using the fractional ``diagonal'' allocation $x_t = X\up{t}_t = \E[X^*_t \mid \cF_t]$ for each arrival $t$.

\subsection{Warm-up: $\ell_\infty$-Norm Objective}
\label{sec:lb-warmup}

To start, we will show how for the standard $\ell_\infty$-norm setting, we can obtain an $O(\log m)$-competitive algorithm.

\begin{theorem} \label{thm:loadBalMax}
    For online load balancing with $\ell_\infty$-norm objective, there exists an $O(\log m)$-competitive algorithm under correlated stochastic arrivals with the bounded-$\Opt$ assumption.
\end{theorem}

To prove this, we first introduce the following seemingly innocuous lemma, which will be the key identity to prove the above theorem, as well as its generalization to other norms. Intuitively, this lemma allows us to upper-bound the dot product of our algorithm's allocation (given by $\sum_t Y\up{t}_t$) with the supporting dual vector (given by $Z\up{T}$) which comes from taking the norm.

\begin{lemma}\label{lem:dual-transfer}
Let $Y\in \Rp^T$ and $Z \in \Rp$ be random variables that are $\cG_T$ measurable for some filtration $(\cG_t)_{t \in [T]}$. Define the martingales $Y\up{t} = \E[Y \mid \cG_t]$ and $Z\up{t} = \E[Z \mid \cG_t]$. Then
\[
\E\Big[Z\up{T} \cdot \tsum_{t \in [T]} Y\up{t}_t\Big] \,\le\,
\E\,\Big[\Big(\max_{t \in [T]} Z\up{t}\Big) \cdot \tsum_{t \in [T]} Y_t\Big].
\]
\end{lemma}

\begin{proof} 
Since \(Y\up{t}_t\) and $Z\up{t}$ are \(\cG_t\)-measurable, we have
\[
\E\big[Z\up{T} \cdot Y\up{t}_t\big] = \E\big[\E[Z\up{T} \mid \cG_t] \cdot Y\up{t}_t\big] = \E\big[Z\up{t} \cdot  Y\up{t}_t\big] = \E\big[Z\up{t} \cdot \E[Y_t \mid \cG_t]\big] = \E\big[Z\up{t} \cdot Y_t\big]
\]
Therefore
\[
\E\Big[Z\up{T} \cdot \tsum_{t \in [T]} Y\up{t}_t\Big]
=
\E\left[\tsum_{t \in [T]} \big(Z\up{t} \cdot Y_t\big)\right]
\leq
\E\Big[\Big(\max_{t \in [T]} Z\up{t}\Big) \cdot \tsum_{t \in [T]} Y_t\Big],
\]
as claimed.
\end{proof}

With this lemma, we can prove the competitive ratio of the \probmatch{} algorithm for load balancing.

\begin{proof}[Proof of Theorem~\ref{thm:loadBalMax}]
    We will show that the \probmatch{} algorithm is $O(\log m)$-competitive under the bounded $\Opt$ assumption, which gives our theorem. Let $L^\Alg$ and $L^\Opt$ be the load vectors of the \probmatch{} algorithm and the hindsight optimum, respectively, i.e. for each machine $i$,
    \[
    L^\Alg_i = \sum_{t \in [T]} X\up{t}_{ti} p_{ti} \qquad \text{and} \qquad L^\Opt_i = \sum_{t \in [T]} X^*_{ti}\, p_{ti}.
    \]
    Next, let $V = (V_1, \dots, V_m)$ be the indicator vector of the largest coordinate of $L^\Alg$, i.e., $V_i = 1$ if $i = \arg\max_j \{L^\Alg_j\}$ (breaking ties arbitrarily) and $V_i = 0$ otherwise; by construction we have $\ip{V}{L^\Alg} = \|L^\Alg\|_{\infty}$. 

    Now define the vector martingale of $V$ by conditioning on the history up to time $t$, namely $V\up{t} := \E[V \mid \cF_t]$, and apply Lemma~\ref{lem:dual-transfer} to each coordinate $i$ (with $Z = V_i$, $Y_t = X^*_{ti} \,p_{it}$,  and $\cG_t = \cF_t$) to obtain
    \[
    \E\Big[V_i \cdot \tsum_{t \in [T]} X\up{t}_{ti} p_{ti}\Big] \leq \E\Big[\Big(\max_{t \in [T]} V\up{t}_i\Big) \cdot \tsum_{t \in [T]} X^*_{ti}\, p_{ti}\Big],
    \]
    where we used the fact $\E[Y_t \mid \cF_t] = \E[X^*_{ti}\,p_{ti}\mid \cF_t] = p_{ti} \cdot \E[X^*_{ti} \mid \cF_t] = p_{ti} X^{(t)}_{ti}$, since $p_t$ is $\cF_t$-measurable (it is indeed revealed at time $t$).
    
    Summing up over all $i$ gives
    \begin{align}
        \E\,\|L^\Alg\|_\infty
        =
        \E\,\ip{V}{L^\Alg} = \sum_{i \in [m]} \E\Big[V_i \cdot \sum_{t \in [T]} X\up{t}_{ti} p_{ti}\Big] &\le
        \sum_{i \in [m]} \E\Big[\Big(\max_{t \in [T]} V\up{t}_i\Big) \cdot \sum_{t \in [T]} X^*_{ti}\, p_{ti}\Big] \notag
        \\
        &= \sum_{i \in [m]} \E\Big[\Big(\max_{t \in [T]} V\up{t}_i\Big) \cdot L^\Opt_i\Big] \notag \\
        &\leq \E\bigg[\|L^\Opt\|_\infty \cdot \sum_{i \in [m]}\max_{t \in [T]} V\up{t}_i \bigg]\notag\\
        &\leq \widehat{\OPT} \cdot \E\bigg[\Big\|\max_{t \in [T]} V\up{t} \Big\|_1\bigg], \label{eq:loadBalInf}
    \end{align}
    where in the last line we use the guarantee of the optimum estimate $\widehat{\OPT}$. We just need to upper bound this maximum, as we did in previous sections. Let $M := \max_{t \in [T]} V\up{t}$. 
    
    As shown in \Cref{lemma:Cgood}, the $\ell_1$-norm is 1-good as it is a Lov\'asz norm. Since each of the vectors $V\up{t}$ has $\|V\up{t}\|_1 \le 1$ (as an average of the random indicator vector $V$), apply the definition of $C$-goodness with $\e = \frac{1}{m}$ to obtain
    \begin{align}
 \E\,\|M\|_1 \le \E \|(M - \ones/m)^+\|_1 + \|\ones/m\|_1 \le \log m + 1.  \label{eq:goodnessCor} 
    \end{align}
    Plugging this in the previous displayed inequality, we get $\E \|L^\Alg\|_\infty \le O(\log m) \cdot \widehat{\OPT}  \le O(\log m) \cdot \E\,\OPT$. This concludes the proof of Theorem \ref{thm:loadBalMax}.
    \end{proof}


\subsection{Extension to General Norms}\label{sec:lb-general-norms}

In fact, the previous argument works as-is for a much broader class of norms, namely norms whose dual norms are $C$-good. Recall for a monotone norm $\|\cdot\|$ on \(\mathbb{R}_+^m\), its dual norm on \(\mathbb{R}_+^m\) is defined as
\[
\|z\|^* := \max\Bigl\{\langle x,z\rangle : x\in\mathbb{R}_+^m,\ \|x\| \le 1\Bigr\}.
\]

\begin{theorem}[Good dual norms imply competitive ratio]\label{thm:good-dual-online}
 Let $\|\cdot\|$ be a monotone norm on \(\mathbb{R}_+^m\) normalized\footnote{In general, one can convert any normed load-balancing problem into one with a ``normalized'' norm by rescaling the input coordinates of the norm, and adjusting each weight $p_{ti}$ to compensate. However, if the norm is not symmetric, this rescaling does not necessarily preserve $C$-goodness, so we must satisfy our ``$C$-good dual'' condition only after normalizing. If the norm is symmetric, this is not a concern as we are simply scaling the entire norm.} such that $\|e_i\| = 1$ for each standard basis vector $e_1, \dots, e_m$, and assume its dual norm $\|\cdot\|^*$ is \(C\)-good. Then, for the online load balancing problem with norm $\|\cdot\|$ objective, there exists a $2(C\log m + 1)$-competitive algorithm under correlated stochastic arrivals with the bounded-$\Opt$ assumption.  
\end{theorem}
We note that the dual norms of Orlicz norms (in particular $\ell_p$-norms) and symmetric norms are again Orlicz and symmetric, respectively; thus we can leverage the $C$-goodness results for these classes from \Cref{sec:setCover} to obtain competitive ratios. Moreover, the dual norms of ordered norms (which are symmetric Lov\'asz norms) are given by a max of $\topk$-norms, i.e. $\|x\| = \max_{k \in [m]} w_k \|x\|_{\topk}$ for $w_1, \dots, w_m \geq 0$.
This gives the following corollary.

\begin{cor}\label{cor:loadBal}
There is an algorithm for Online Load Balancing with norm $\|\cdot\|$ under the correlated stochastic online model with the following competitive ratios: 
\begin{itemize}
    \item $2(\log m + 1)$-competitive if the norm $\|\cdot\|$ is a max of Top-$k$ norms.\item $4(\log m + 1)$-competitive if the norm $\|\cdot\|$ is an Orlicz norm.
    \item $4(\log m + 1)^2$-competitive if the norm $\|\cdot\|$ is symmetric.
\end{itemize}
\end{cor}

\begin{proof}[Proof of Theorem \ref{thm:good-dual-online}]
Again with the bounded $\Opt$ assumption, we can assume that we have an optimum estimate $\widehat{\OPT}$ such that $\Opt \leq \widehat{\OPT} \leq 2\Opt$ for every scenario. After this reduction, we apply the \probmatch{} algorithm.

Recall our notation: the \probmatch{} algorithm uses the fractional assignment vector $X\up{t}_t = \E[X^*_t \mid \cF_t] \in \Delta_m$ for the $t$-th job, incurring load vector $L^\Alg$ with coordinates $ L^\Alg_i = \sum_{t \in [T]} X\up{t}_{ti} p_{ti}$, while $\OPT$'s load vector $L^*$ is given by $L^\Opt_i = \sum_{t \in [T]} X^*_{ti}\, p_{ti}.$

Similar to the previous part, to upper bound the cost of the algorithm $\|L\|$ we first represent it using duality, namely note that for every vector $x \in \R^m_+$, there is a vector $z$ with $\|z\|^* = 1$ such that $\|x\| = \ip{x}{z}$ (since we are in finite dimensions, the ``double-duality'' property holds for our norm).
Therefore, let $V \in \R^m_+$ be a random vector with $\|V\|^* = 1$ such that $\|L\| = \ip{V}{L^\Alg}$ in every scenario. Note that $V_i \leq 1$ for each $i$, since $1 = \|V\|^* \geq \|V_i e_i\|^* = \frac{V_i}{\|e_i\|} = V_i$.

Again, define the martingale $V^{(t)} := \E[V \mid \cF_t]$ and use the exact same development as in \eqref{eq:loadBalInf} to upper bound the cost of the \probmatch{} algorithm as 
\begin{align*}
    \E\,\|L\|
    =
    \E\,\ip{V}{L} = \sum_{i \in [m]} \E\Big[V_i \cdot \sum_{t \in [T]} X\up{t}_{ti} p_{ti}\Big] &\stackrel{\text{Lem~\ref{lem:dual-transfer}}}{\le}
    \sum_{i \in [m]} \E\Big[\Big(\max_{t \in [T]} V\up{t}_i\Big) \cdot L^\Opt_i\Big]\\
    &~~~\le \E\bigg[ \Big\| \max_{t \in [T]} V\up{t} \Big\|^*\,\Big\| L^\Opt \Big\|\bigg]\\
    &~~~\le \widehat{\OPT} \cdot 
    \E\, \Big\| \max_{t \in [T]} V\up{t} \Big\|^*,
\end{align*}
where the next to last inequality follows from generalized Cauchy-Schwarz inequality $\ip{x}{z} \le \|x\| \|z\|^*$, and the last inequality uses the guarantee of the optimum estimate $\widehat{\OPT}$. Moreover, by the assumption that the dual norm $\|\cdot\|^*$ is $C$-good and the fact that $V \in [0,1]^m$, we can use an analog of the triangle inequality argument from \eqref{eq:goodnessCor} to upper-bound the RHS maximum as
\[
\E\big[\| \max_{t \in [T]} V\up{t}\|^*\big] 
~\le~
\E \|(\max_{t \in [T]} V\up{t} - \ones/m)^+\|^* + \|\ones/m\|^* 
~\le~ 
\E\|V\|^* \cdot C\log m + \frac{\|\mathbf{1}\|^*}{m}
~\le~ C\log m + 1,
\]
where we use $\|V\|^* = 1$ and $\|\mathbf{1}\|^* \leq \sum_i \|e_i\|^*=\sum_i \frac{1}{\|e_i\|} = m$. Using again the estimate guarantee $\widehat{\OPT} \le 2 \OPT$, we obtain 
\[
 \E\,\|L^\Alg\| \le (C \log m + 1)\cdot \E\,\widehat{\OPT} \le 2(C \log m + 1)\cdot \E\OPT\,.  \qedhere
\]
\end{proof}

\section{Weighted Paging, MTS on Star Metrics, and Ski-Rental} \label{sec:pagingMTS}
    In this section, we discuss state-based online problems with natural fractional formulations, including weighted paging, MTS on star metrics, and ski-rental. For weighted paging and MTS on star metrics, known rounding reductions convert the fractional algorithms into randomized integral algorithms with the same asymptotic guarantees. For ski-rental, posterior matching directly yields a randomized algorithm.

\subsection{Weighted Paging}
In the Fractional Weighted Paging problem, we have a cache of size $k$ and a set of $n$ pages. At each time $t \in [T]$, a page $p_t \in [n]$ is requested. The algorithm must maintain a fractional cache vector \(y^t \in [0,1]^n\) over all pages such that $\sum_i y^t_i \leq k$ at each time $t$ and $y^t_{p_t} = 1$. Our objective will be to minimize the eviction cost: When a coordinate is decreased, i.e. $y^t_i < y^{t-1}_i$, the algorithm pays eviction cost $w_i (y^{t-1}_i - y^t_i)$. It has been observed (e.g. see \cite{BBN-JACM12}) that---up to an additive constant that is independent of the sequence length---this is equivalent to cost for fetching pages.

Additionally, we will use the standard approach (e.g. see \cite{BBN-JACM12}) of examining the complementary variables $x_{i,j}$, which denote the total amount of page $i$ evicted from the cache after the $j$th time it was requested, and before the $(j+1)$th time. At each time $t$, let $j(i,t)$ denote the number of times item $i$ has been requested up to time $t$. Let $B(t) = \{p_s : s \leq t\}$ denote the set of pages requested up to time $t$. Then our problem can be described by the following program
\begin{align*}
    \min ~&~ \sum_{i,j} w_i x_{i,j}\\
    \text{s.t.} ~&~ \sum_{i \in B(t) \setminus p_t} x_{i,j(i,t)} \geq |B(t)| - k && \forall t \in [T]\\
    &~ 0 \leq x_{i,j} \leq 1 && \forall i \in [n],~j \geq 0
\end{align*}

Let $X^*$ be the optimal solution to this program, as a function of the random sequence arrivals $(p_1, \dots, p_T)$. The offline optimum of minimizing the eviction cost is given by $\sum_{i,j} w_i X^* _{i, j}$.

We can now describe our algorithm as a capped \probmatch{} algorithm, similar to \Cref{alg:capped-pm-set-cover} for Set Cover. Fix a page $i$, and suppose $i$ is requested for the $j$th time at time $r(i,j)$. For each time $t$ from $r(i,j)$ until page $i$ is requested again (i.e. $r(i,j) \leq t < r(i,j+1)$), define the martingale $X\up{t}_{i,j} := \E[X^*_{i,j} \mid \cF_t]$, and define the running maximum $M\up{t}_{i,j} := \max_{r(i,j) \leq s \leq t}\{X\up{s}_{i,j}\}$.

Our algorithm will seek to match the running maximum $M\up{t}_{i,j}$ with its evictions. However, like in set cover, we will need to cap off small evictions to apply our martingale lemmas effectively. To this end, we instead ensure the fraction of page $i$ evicted at time $t$ is $2\left(M\up{t}_{i,j} - \frac{1}{2k}\right)^+$. Our algorithm will ultimately be as follows in \Cref{alg:paging}, phrased in terms of the variables $y^t$.

\begin{algorithm}
    \caption{Capped \ProbMatch{} for Paging}\label{alg:paging}
    For each time $t$, when $p_t$ is requested:
    \begin{enumerate}
        \item Set $y\up{t}_{p_t} = 1$.
        \item For every page \(i\notin B(t)\), set \(y^t_i=0\).
        \item For every page \(i\in B(t)\setminus\{p_t\}\), compute
    \(M^{(t)}_{i,j(i,t)}\) and set
        \[
        y^{t}_i = \max\left\{0,~ 1 - 2\left(M\up{t}_{i,j(i,t)} - \frac{1}{2k}\right)^+\right\}.
        \]
    \end{enumerate}
\end{algorithm}

\begin{theorem} \label{thm:paging}
    \Cref{alg:paging} is $2\log(2k)$-competitive for weighted paging under correlated stochastic online model.
\end{theorem}
\begin{proof}
    We seek to show both that \Cref{alg:paging} returns a feasible solution to the paging problem, and that it incurs a cost of at most $O(\log k) \cdot \Opt$. We will defer the proof of feasibility to \Cref{sec:proofPagingFeasibile}, as it follows from standard methods of \cite{BCLLM-STOC18} and is tangential to our main techniques.
    \begin{claim}\label{claim:paging-feasibility}
        For the $y^t_i$ returned by \Cref{alg:paging}, we have $\sum_{i \in [n]} y^t_i \leq k$ for each time $t \in [T]$.
    \end{claim}
    Assuming this claim, we focus instead on bounding the cost of our algorithm.

    Consider a page $i \in [n]$ and a phase $j \geq 1$ of $i$. Note that since $M\up{t}_{i,j}$ is monotone increasing in this phase, the allocation $y\up{t}_i$ is monotone decreasing. Thus, by the end of this phase at time $\tau := r(i, j+1)-1$, the algorithm has incurred eviction cost at most $w_i\cdot (1 - y^{\tau}_i) \leq w_i \cdot 2(M\up{\tau}_{i,j} - \frac{1}{2k})^+$. Hence, the expected cost of this phase of page $i$, conditioned on the filtration at the start of the phase $\cF_{r(i, j)}$, is
    \[
    \E\Big[w_i\cdot (1 - y^{\tau}_i) \,\Big|\, \cF_{r(i, j)}\Big] 
    ~\leq ~ \E\Big[2w_i(M\up{\tau}_{i,j} - \tfrac{1}{2k})^+ \,\Big|\, \cF_{r(i, j)}\Big] 
    ~\leq~ 2w_i\log(2k) \cdot \E\big[X\up{\tau}_{i,j} \mid \cF_{r(i, j)}\big],
    \]
    where the second inequality is by applying \Cref{lem:max-martingale} to the stopped martingale $X\up{r(i,j)}_{i,j},\ldots,X\up{\tau}_{i,j}$.  

    Taking expectations, we have 
    \[
    \E\big[w_i\cdot \big(1 - y^{r(i,j+1)-1}_i\big)\big] \leq 2\log(2k) \cdot \E[w_i \cdot X^*_{i,j}] \enspace .
    \]
    Summing over all $i$ and $j$, we have
    \[
    \E[\Alg] = \sum_{i,j}\E\big[w_i\cdot \big(1 - y^{r(i,j+1)-1}_i\big)\big] \leq 2\log(2k) \cdot \sum_{i,j}\E[w_i \cdot X^*_{i,j}] = 2\log(2k) \cdot \E[\Opt] \enspace . \qedhere
    \]
\end{proof}

This proves the desired guarantees for~\Cref{alg:paging}, thus proving~\Cref{thm:paging}.

\subsection{Metrical Task System on Star Metrics}
In the Metrical Task System (MTS) problem, we maintain the position of an agent in a metric space $\cM$ over $[n]$, who must move to service a sequence of incoming requests $r^1, \dots, r^T$ while minimizing costs. Each request $r^t$ specifies a vector of service costs $r^t_1, \dots, r^t_n$ over $[n]$. Upon the arrival of each request, the algorithm must determine how to move the agent in $\cM$, after which we pay both the movement distance as well as the service cost given by $r^t$.

In fractional MTS, we formally maintain a vector $x^t \in \Delta_n$, where $x^t_i$ denotes the fractional amount of the agent at location $i$ at time $t$. When request $r^t$ arrives, the algorithm picks $x^t \in \Delta_n$, and the cost of step $t$ is given as $W_\cM(x^{t-1}, x^t) + \ip{x^t}{r^t}$, where $W_\cM$ denotes the Wasserstein distance with respect to metric $\cM$.

We obtain results for fractional MTS when the metric $M$ is a star metric, i.e. $d_\cM(i,j) = d_i + d_j$ for some $d_1, \dots, d_n \geq 0$. For this problem, we first use the reduction of \cite{BBN-JACM12}, which allows us to focus on a \emph{reduced setting} of MTS defined as follows.

\begin{definition}
    Let $\cM$ be the star metric with $d_\cM(i,j) = d_i + d_j$ for some $d_1, \dots, d_n \geq 0$. The \emph{reduced MTS problem} with metric $\cM$ is as follows. 
    
    Request vectors $r^t \in \Rp^n$ arrive continuously over time $t \in [0, T]$, such that $r^t$ is constant for each unit interval $t \in [k, k+1)$, where $k \in \N$. At each point in time, the algorithm picks $x^t \in \Rp^n$ subject to the following constraints.

    \begin{enumerate}
        \item Each coordinate $x^t_i$ is non-decreasing in $t$, except immediately following times $\tau_0(i), \tau_1(i), \tau_2(i), \dots$ which are defined by the property \[
        \int_{\tau_{j-1}(i)}^{\tau_{j}(i)} r^t_i\,\d t \,=\, d_i.
        \]
        In other words, each $\tau_j(i)$ marks the end of the $j$th \emph{phase} of location $i$ (i.e. the interval $(\tau_{j-1}(i), \tau_{j}(i)]$), where the total service costs in each phase equals $d_i$.
        \item We must have $\sum_ix^t_i \geq 1$ for all $t$ (notably, we don't enforce equality as larger $x$ will only cost more).
        \item We incur cost $d_i$ per unit mass in location $i$ at the end of each phase of $i$. Hence, we seek to minimize
        \[\sum_{i \in [n]} d_i \cdot \sum_j x^{\tau_j(i)}_i.
        \]
    \end{enumerate}
\end{definition}

\begin{lemma}[Lemmas 6.1 and 6.2 in \cite{BBN-JACM12}]\label{lem:MTS-reduction}
    Any $c$-competitive algorithm for the reduced MTS problem with star metric $\cM$ can be converted into a $4c$-competitive algorithm for the original MTS problem with star metric $\cM$.
\end{lemma}

\medskip\textbf{\ProbMatch{}.} For the reduced MTS problem, our algorithm is as follows. Let $Z^t \in \Delta_n$ denote that hindsight optimum's distribution over the $n$ locations at $t$. For each location $i$, and for any time $t$ in phase $j$ of that location, let $X\up{t}_{i,j} := \E[Z^{\tau_j(i)}_i \mid \cF_t]$ denote the expected total mass in the hindsight optimum in this location at the end of phase $j$, given the filtration $\cF_t$. For convenience, we will say $\tau_j(i) = T+1$ if location $i$ has fewer than $j$ phases, and $Z^{T+1} = \mathbf{0}$.

With the \probmatch{} paradigm, we want to let $x^t_i$ follow $X\up{t}_{i,j}$ as closely as possible within phase $j$. But since $x^t_i$ cannot decrease within the phase, we instead follow the running max 
\[
M\up{t}_i := \sup_{s \in (\tau_{j(i,t) - 1}(i),\, t]} X\up{s}_{i,j(i,t)},
\]
where $j(i,t)$ is the index of the phase of location $i$ at time $t$. However, like Set Cover and Paging, we will also need to cap $x^t_i$ to a minimum value to get our desired competitive bound. Hence, our final algorithm will be to set
\[
x^t_i = 2\left(M\up{t}_i- \frac{1}{2n}\right)^+
\]
at each time $t$.

\begin{theorem}\label{thm:MTS}
    For the reduced MTS problem with star metric, this capped \probmatch{} algorithm is $2\log(2n)$-competitive in the correlated stochastic online model.
\end{theorem}
Together with \Cref{lem:MTS-reduction}, this implies that there exists an $8\log(2n)$-competitive algorithm for MTS with star metric under correlated stochastic online model.

\begin{proof}[Proof of \Cref{thm:MTS}]
    First, we check that our algorithm is feasible. First, by definition, it maintains non-decreasing mass $x^t_i$ for $t$ within a phase. Now we argue that $\sum_i x^t_i \ge 1$ for all times $t$. Note that, by feasibility of the optimal solution $Z^t$, we have 
    \[
     \sum_{i \in [n]} M\up{t}_i \geq \sum_{i \in [n]} X\up{t}_{i,j(i,t)} = \E\Big[\sum_{i \in [n]} Z^{\tau_{j(i,t)}(i)}_i \,\Big|\, \cF_t\Big] \ge \E\Big[\sum_{i \in [n]} Z^{t}_i \,\Big|\, \cF_t\Big]  \geq 1,
    \]
    where the next-to-last inequality is because the optimal solution to the reduced MTS problem also needs to maintain non-decreasing mass $Z_i^t$ within a phase of $i$, and by definition $t$ is within the phase ending at time $\tau_{j(i,t)}(i)$. Hence, we have
    \[
    \sum_{i \in [n]} x^t_i = 2\sum_{i \in [n]}\Big(M\up{t}_i - \frac{1}{2n}\Big)^+ \geq 2(1 - 1/2) = 1,
    \]
    proving the feasibility of our algorithm. 

    Next, we bound the algorithm's cost. To do so, fix a location $i$ and a phase index $j$. For this (possible) phase, $\Opt$ pays cost 
    \[
    \Opt(i,j) := \E[d_i \cdot Z^{\tau_j(i)}_i] 
    = 
    \E[d_i \cdot \E[Z^{\tau_j(i)}_i \mid \cF_{\tau_j(i)}]]
    =
    \E[d_i \cdot X\up{\tau_j(i)}_{i,j}].
    \]
    Meanwhile, our algorithm pays cost
    \[
    \Alg(i,j) := \E\big[d_i \cdot x^{\tau_j(i)}_i\big] = \E\Big[d_i \cdot 2\Big(M\up{\tau_j(i)}_i - \frac{1}{2n}\Big)^+\Big].
    \]
    By \Cref{lem:max-martingale} applied conditionally to the stopped martingale $X\up{t}_{i,j}$ over $t \in (\tau_{j-1}(i), \tau_j(i)]$, we have \[
    \E\bigg[\bigg(M\up{\tau_j(i)}_i - \frac{1}{2n}\bigg)^+\,\bigg|\, \cF_{\tau_{j-1}(i)}\bigg] \leq \log(2n)\cdot \E\big[X\up{\tau_j(i)}_{i,j} \mid \cF_{\tau_{j-1}(i)}\big].
    \]
    Taking expectations, we see that $\Alg(i,j) \leq 2\log(2n) \cdot \Opt(i,j)$. Summing over all locations $i$ and phases $j$ gives $\Alg \leq 2\log(2n) \cdot \Opt$, as desired.
\end{proof}

\subsection{Ski-Rental}
In the (continuous time) ski-rental problem, there is some time horizon $T \in \Rp$ unknown to the algorithm, representing the length of the skiing season. The algorithm must choose, at each time $t \in [0,T]$, whether to \emph{rent} skis (incurring cost $1$ per unit time spent renting), or to \emph{buy} skis (incurring a one-time cost $B > 0$, but no longer incurring rental costs). Hence, if the algorithm chooses to buy at time $t^\Alg \leq T$, the algorithm incurs cost $t^\Alg + B$. If the algorithm never buys, it incurs cost $T$. It is easy to see that the hindsight optimum is $\Opt = \min\{T, B\}$, as the optimum solution will buy at time $0$ if $B \leq T$ and otherwise will rent for the entire season.

We can represent a randomized algorithm for ski-rental by its buying probability function $x : \Rp \to [0,1]$, where $x$ is a monotone increasing function such that $x(t)$ represents the cumulative probability of having bought skis by $t$ if the season has not yet ended (i.e. if $T \geq t$). The cost of such an algorithm for a given time horizon $T$ is then
\[
\cost_x(T) = \int_0^T(1-x(t))\,\d t + B\cdot x(T) .
\]

\medskip\textbf{\ProbMatch{}.} In our stochastic setting, we assume $T \sim \cD$ for some known distribution $\cD$. In this setting, we can naturally define a \probmatch{} algorithm: At time $t$, if the ski season has not ended, we increase our cumulative probability of buying to match the current conditional probability that $\Opt$ bought skis. In other words, we maintain the buying probability function 
\[
x(t) = \Pr[\Opt \text{ bought skis} \mid \cF_t] = \Pr[T \geq B \mid T\geq t]. 
\]
However, in order to optimize our competitive ratio, we again run a capped version of \probmatch{}: we set $x(t) = 0$ for $t \leq t_0$, where $t_0$ is a threshold time we will optimize later. Thus, our final algorithm will be given by
\[
x(t) = \begin{cases}
    0 & t \in [0,t_0],\\
    \frac{\Pr[T \geq B]}{\Pr[T \geq t]} & t \in (t_0, B],\\
    1 & t > B.
\end{cases}
\]
\begin{theorem}\label{thm:ski-rental}
    For any $\cD$, there exists a $t_0 \geq 0$ such that the above algorithm is $\frac{e}{e-1}$-competitive for ski-rental with stochastic $T \sim \cD$.
\end{theorem}
\begin{proof}
    First, we will compute the expected cost incurred by our algorithm. Let $R(t) := \Pr[T \geq t]$ be the tail probability for each time $t$. For convenience, we will assume that $R(t)$ is differentiable (this can be achieved by adding vanishingly small noise to the distribution of $T$). 
    If \(R(B)=0\), then \(T<B\) almost surely and the algorithm that never buys before \(B\) is optimal. Thus assume \(R(B)>0\).
    We have
    \begin{align*}
        \E_{T \sim \cD}[\cost_x(T)]
        &=\E\left[\int_0^T(1-x(t))\,\d t\right] + B \cdot \E[x(T)]\\
        &= \left(\int_0^\infty R(t)\cdot (1-x(t)) \,\d t \right)+ B \cdot \left(\int_0^\infty (-R'(t))\cdot x(t)\, \d t\right) .
    \end{align*}
    Using our choice of $x(t)$, we have
    \begin{align}
        E[\cost_x(T)] &=  \left(\int_0^B R(t)\cdot (1-x(t)) \,\d t \right)+ B \cdot \left(\int_0^B (-R'(t))\cdot x(t)\, \d t + \int_B^\infty (-R'(t))\, \d t\right) \notag\\
        &= \left(\int_0^B R(t)\, \d t - \int_{t_0}^B R(t) \cdot \frac{R(B)}{R(t)}\,\d t\right) + B\cdot \left(\int_{t_0}^B \frac{-R'(t) \cdot R(B)}{R(t)}dt + \int_B^\infty(-R'(t))dt \right)\nonumber.
    \end{align}
     Now using the fact that $\Opt = \E[\min\{T, B\}] = \int_0^B R(t)\,\d t$, we have
    \begin{align}
         E[\cost_x(T)]  &= \Opt -(B-t_0) \cdot R(B) + B\cdot R(B)\cdot \log\left(\frac{R(t_0)}{R(B)}\right) + B\cdot R(B)\nonumber\\
        &= \Opt + R(B) \cdot \left(t_0+ B \cdot \log\left(\frac{R(t_0)}{R(B)}\right)\right) \enspace .\label{eq:ski-rental-bound}
    \end{align}
    To optimize our choice of $t_0$, we let 
    \[
    t_0 = \arg \min_{t \in [0, B]}\{t + B\log(R(t)/R(B))\}\qquad \text{and} \qquad m := t_0 + B \log(R(t_0)/R(B)).
    \]
    Notice that this choice of $t_0$ ensures that for all $t \in [0,B]$, we have
    \begin{align*}
        m
        \leq t + B\log\left(\frac{R(t)}{R(B)}\right) \implies R(B) \cdot e^{\frac{m-t}{B}} \leq R(t).
    \end{align*}
    Integrating the above inequality from $0$ to $B$ gives
    \[
    \Opt = \int_0^B R(t)dt \geq R(B) \cdot  \int_0^B e^{\frac{m-t}{B}}dt = R(B) \cdot Be^{\frac{m}{B}-1} \cdot (e-1) \geq R(B) \cdot m\cdot (e-1),
    \]
    where the last inequality comes from $e^{m/B - 1} \geq \frac{m}{B}$. Hence, we get the bound $R(B) \cdot m \leq \frac{1}{e-1}\Opt$, which we can substitute into \eqref{eq:ski-rental-bound} and recall the definition of $m$ to get
    \[
    \Alg = \E[\cost_x(T)] = \Opt + R(B) \cdot m \leq \Opt + \frac{1}{e-1}\Opt = \frac{e}{e-1}\Opt,
    \]
    as desired.
\end{proof}

\section*{AI Disclosure}
ChatGPT 5.3 was used during the research process to help brainstorm possible proof strategies. All results, proofs, and exposition were written, verified, and adapted by the authors.

\bigskip
\appendix
\noindent {\LARGE \bf Appendix}
  \section{A Minimax Principle}
\label{sec:minimax}

We state the result for minimization problems. The maximization version is obtained by reversing the relevant inequalities (equivalently, by applying the theorem to the negative of the objective). The formulation below allows the feasible action at a round to depend on both the revealed input prefix and the algorithm's previous actions.

{
Our online model can be viewed as a compact-action variant of the
request--answer-game framework of 
\cite{BenDavidBKTTW}. Their framework allows arbitrary trajectory
costs but assumes a finite answer set, whereas we allow compact
continuous action spaces and explicitly incorporate
input- and history-dependent feasibility constraints, as required
by fractional online problems.}

\paragraph{Online model.}
Fix a time horizon $T\in\mathbb N$. Inputs of length less than $T$ can be padded with a terminal symbol and dummy rounds that impose no further constraint or cost, so we may represent every valid input by a length-$T$ sequence. Let $\cX_{step}$ (which will model each step of the input) be any set, and let 
\[
    \cX\subseteq\cX_{\mathrm{step}}^T
\]
be the set of valid (padded) inputs.

Let $\cY_{\mathrm{step}}$ be a compact metric space (which will model the ground set of actions the algorithm can take at each time). At time $t$, after observing $x_{\le t} = (x_1,\ldots,x_t) \in \cX_{step}^t$ and having played $y_{<t} = (y_1,\ldots,y_{t-1}) \in \cY_{step}^{t-1}$, the algorithm must choose an action
\[
    y_t\in \Gamma_t(x_{\le t},y_{<t})\subseteq \cY_{\mathrm{step}},
\]
where $\Gamma_t$ describes the feasible actions at that history. We assume that this set is nonempty whenever $y_{<t}$ is a feasible action prefix for $x_{<t}$. For a complete input $x\in\cX$, define the set of feasible action sequences
\[
    \cY(x)
    :=
    \left\{
        y=(y_1,\ldots,y_T)\in\cY_{\mathrm{step}}^T:
        y_t\in\Gamma_t(x_{\le t},y_{<t})
        \text{ for every }t\in[T]
    \right\}.
\]
We assume that $\cY(x)$ is compact for every input $x$. In the applications of this paper, this follows immediately because the actions lie in a compact box or simplex and feasibility is described by closed linear inequalities.

For example, in online fractional set cover, $y_t\in[0,1]^m$ is the current set-selection vector and
\[
    \Gamma_t(x_{\le t},y_{<t})
    =
    \Big\{
        z\in[0,1]^m:
        z\ge y_{t-1}
        \text{ and }
        \sum_{i:x_t\in S_i}z_i\ge 1
    \Big\},
\]
where $x_t$ indicates the element to be covered at time $t$ (and writing $y_0 := 0$). For fractional online load balancing, one simply has $\Gamma_t(x_{\le t}, y_{<t}) =\Delta_m$.

Let
\[
    \cost:
    \{(x,y):x\in\cX,\ y\in\cY(x)\}
    \longrightarrow \mathbb R
\]
be the cost function. We assume that, for every fixed input $x$, the map $y\mapsto\cost(x,y)$ is continuous on the compact set $\cY(x)$. Let $\OPT:\cX\to\mathbb R$ denote the offline optimum.

\paragraph{Representation of randomized online algorithms.}
For a compact metric space $K$, let $\mathcal P(K)$ denote the set of Borel probability measures on $K$. We represent a randomized online algorithm by a family of distributions $\nu_x$, one for each input $x$ (giving the distributions of the whole action sequence for that input), but satisfying the online condition that when two input sequences $x,x'$ agree up to some time $t$, their corresponding distributions $\nu_x$ and $\nu_{x'}$ must give 
the same probabilities for the actions in the first $t$ times. More precisely, a randomized online algorithm is given by a family of distributions 
\[
    \nu=(\nu_x)_{x\in\cX},
    \qquad
    \nu_x\in\mathcal P(\cY(x)),
\]
such that for every $x,x' \in \cX$ and $t \in [T]$,
\begin{align}
    x_{\le t}=x'_{\le t}
    \quad\Longrightarrow\quad
    (\pi_t)_\#\nu_x=(\pi_t)_\#\nu_{x'}\,;
    \label{eq:causality}
\end{align}
Here $\pi_t(y_1,\ldots,y_T):=(y_1,\ldots,y_t)$ denotes the
projection onto the first $t$ actions, and $(\pi_t)_\#\nu_x$ is the
corresponding marginal distribution; namely,
\[
(\pi_t)_\#\nu_x(U):=\nu_x(\pi_t^{-1}(U))
\qquad
\text{for every Borel set }U\subseteq\cY_{\mathrm{step}}^t.
\]

We remark that this representation is equivalent to the usual round-by-round description of a randomized online algorithm. First, it is clear that any ``round-by-round'' randomized algorithm induces a distribution $(\nu_x)_{x\in\cX}$ that satisfies the online condition \eqref{eq:causality}. Conversely, a collection of distributions $\{\nu_x\}_{x \in \cX}$ satisfying the online condition \eqref{eq:causality} defines a unique joint distribution of the first $t$ actions based only on the input prefix $x_{\le t}$, and conditioning this distribution on the first $t-1$ actions $y_{< t}$ gives the distribution of the action $y_t$ at round $t$ (such conditional distributions exist because $\cY_{\mathrm{step}}$ is a compact metric space; see e.g., \cite{Kallenberg2002} or \cite{Williams1991}). 

\medskip

Given these formal definitions of the online model and randomized online algorithms, we can finally state the key minimax theorem that we will prove. 

\begin{theorem}[Minimax for online algorithms with a fixed horizon]
\label{thm:onlineMinimaxFull}
Fix $T\in\mathbb N$ and $\alpha,\beta\ge 0$. Suppose that, for every finitely supported probability distribution $\cD$ over $\cX$, there exists a randomized online algorithm $\nu^{\cD}$ such that
\begin{align}
    \E_{X\sim\cD}
    \E_{Y\sim\nu_X^{\cD}}
    [\cost(X,Y)]
    \le
    \alpha\,\E_{X\sim\cD}[\OPT(X)]+\beta\,.
    \label{eq:bayesian-guarantee}
\end{align}
Then there exists a single randomized online algorithm $\nu^\star$ such that, for every input $x\in\cX$,
\begin{align}
    \E_{Y\sim\nu_x^\star}[\cost(x,Y)]
    \le
    \alpha\OPT(x)+\beta\,.
    \label{eq:worst-case-guarantee}
\end{align}
\end{theorem}

We prove this theorem in the remainder of the section. Let $\mathfrak A$ be the set of all randomized algorithms:
\begin{align*}
    \mathfrak A := \bigg\{ (\nu_x)_{x\in\cX} \in \prod_{x \in \cX} \mathcal{P}(\cY(x)) : \textrm{ \eqref{eq:causality} holds for all $x,x' \in \cX$ and $t \in [T]$}\bigg\}.
\end{align*}
A key property is that this set is compact. More precisely, for each $x \in \cX$, endow the set of distributions $\mathcal{P}(\cY(x))$ with the topology of weak convergence (i.e., weak-*-topology), and endow the space $\prod_{x \in \cX} \mathcal{P}(\cY(x))$ with the product topology. 

\begin{lemma} \label{lemma:Acompact}
    $\mathfrak A$ is a compact subset of $\prod_{x \in \cX} \mathcal{P}(\cY(x))$.
\end{lemma}

\begin{proof}
For each $t \in [T]$ and each pair of inputs $x, x' \in \mathcal{X}$ with $x_{\le t} = x'_{\le t}$, the set enforcing the online constraint \eqref{eq:causality}
$$C_{x,x',t} := \bigg\{ (\nu_y)_{y\in\cX} \in \prod_{y\in\cX}\mathcal P(\cY(y)) : (\pi_t)_\# \nu_x = (\pi_t)_\# \nu_{x'} \bigg\}.$$ So the set $\mathfrak A$ is just  $\bigcap_{x,x',t} C_{x,x',t}$, with $x,x',t$ ranging over the appropriate elements. 

For each $x$, since $\cY(x)$ is assumed to be a compact metric space, the space of probability measures $\mathcal{P}(\cY(x))$ is compact (by Prokhorov's Theorem, see e.g.,~\cite{Dudley2002}). Therefore, by Tychonoff's Theorem, the space $\prod_{x \in \cX} \mathcal{P}(\cY(x))$ is also compact. Thus, to prove the lemma it suffices to show that each of the sets $C_{x,x',t}$ is closed: this implies that the intersection $\bigcap_{x,x',t} C_{x,x',t}$ is a closed subset of a compact space (namely $\prod_{x \in \cX} \mathcal{P}(\cY(x))$), which is always compact.

So fix $x,x'\in\cX$ and $t\in[T]$ such that
$x_{\le t}=x'_{\le t}$. Define
\[
F_{x,x',t}:
\prod_{z\in\cX}\mathcal P(\cY(z))
\longrightarrow
\mathcal P(\cY_{\mathrm{step}}^t)
\times
\mathcal P(\cY_{\mathrm{step}}^t)
\]
by
\[
F_{x,x',t}\big((\nu_z)_{z\in\cX}\big)
:=
\big((\pi_t)_\#\nu_x,(\pi_t)_\#\nu_{x'}\big),
\]
as well as the diagonal 
\[
\Delta
:=
\{(\mu,\mu):
\mu\in\mathcal P(\cY_{\mathrm{step}}^t)\}.
\]
We can express the set of interest as $C_{x,x',t} = F^{-1}_{x,x',t}(\Delta)$. 

The map $F_{x,x',t}$ is continuous. Indeed, if
$\mu_n\to\mu$ weakly in $\mathcal P(\cY(x))$, then for every bounded
continuous function
$f:\cY_{\mathrm{step}}^t\to\mathbb R$,
\[
\int f\,d((\pi_t)_\#\mu_n)
=
\int (f\circ\pi_t)\,d\mu_n
\longrightarrow
\int (f\circ\pi_t)\,d\mu
=
\int f\,d((\pi_t)_\#\mu),
\]
since $f \circ \pi_t$ is itself a bounded continuous function. 

Moreover, since $\cY_{\mathrm{step}}^t$ is a metric space, $\mathcal P(\cY_{\mathrm{step}}^t)$ is Hausdorff under the
weak topology \cite{Kallenberg2002}, so its diagonal $\Delta$ is closed. 

Combining these two observations, we get that the pre-image $C_{x,x',t} = F^{-1}_{x,x',t}(\Delta)$ is closed, as desired. This concludes the proof of the compactness of $\mathfrak A$.  
\end{proof}

The set $\mathfrak A$ is nonempty by the hypothesis of
\Cref{thm:onlineMinimaxFull} applied to any point-mass prior, and it
is convex: if $\nu,\nu'\in\mathfrak A$ and $\lambda\in[0,1]$, then
$\lambda\nu+(1-\lambda)\nu'\in\mathfrak A$.

We first show that the minimax result of \Cref{thm:onlineMinimaxFull} holds for a finite collection of inputs. For that, fix a finite set of inputs $\cX' =\{x^1,\ldots,x^k\}\subseteq\cX$.
For each input $x\in\cX$ and algorithm $\nu = (\nu_x)_{x \in \cX} \in\mathfrak A$, define the excess cost
\[
    g_x(\nu)
    :=
    \E_{Y\sim\nu_x}[\cost(x,Y)]
    -\alpha\OPT(x)-\beta.
\]
Since $\cost(x,\cdot)$ is continuous on $\cY(x)$ (hence also bounded, since $\cY(x)$ is compact), the function $g_x$ is continuous in $\nu$. It is also affine in $\nu$.

Now consider the achievable excess-cost vectors over the finite set of instances $\cX'$:
\[
    V_{\cX'}
    :=
    \left\{
        \bigl(g_{x^1}(\nu),\ldots,g_{x^k}(\nu)\bigr):
        \nu\in\mathfrak A
    \right\}
    \subseteq\mathbb R^k.
\]

From \Cref{lemma:Acompact}, $\mathfrak A$ is compact, and it can be readily verified that it is convex as well; since $g_{x^j}$ is continuous and affine, these imply that $V_{\cX'}$ is a compact convex subset of $\mathbb R^k$. 

Therefore, we can apply the standard finite-dimensional minimax theorem (von Neumann's minimax theorem; see e.g., \cite{Rockafellar1970}) to the bilinear payoff $\langle q,v\rangle$ on the compact convex sets $V_{\cX'}$ and the simplex $\Delta_k:=\{q\in\mathbb R_+^k:\sum_{j=1}^k q_j=1\}$ to get
\begin{align}
    \min_{v\in V_{\cX'}}\max_{q\in\Delta_k}\langle q,v\rangle
    =
    \max_{q\in\Delta_k}\min_{v\in V_{\cX'}}\langle q,v\rangle\,.
    \label{eq:finite-minimax}
\end{align}

Fix any $q\in\Delta_k$ and regard it as the finitely supported distribution that assigns probability $q_j$ to input $x^j$. By the hypothesis in our minimax theorem, namely \eqref{eq:bayesian-guarantee}, there exists an algorithm $\nu^q\in\mathfrak A$ such that
\[
    \sum_{j=1}^k q_j g_{x^j}(\nu^q)
    \le 0\,.
\]
Equivalently, if $v(\nu^q)\in V_{\cX'}$ is the corresponding excess-cost vector, then
    $\min_{v\in V_{\cX'}}\langle q,v\rangle\le 0$.
This holds for every $q\in\Delta_k$, so the right-hand side of \eqref{eq:finite-minimax} is at most zero, and the left-hand side gives that there exists $v^{\cX'}\in V_{\cX'}$ such that
    $\max_{q\in\Delta_k}\langle q,v^{\cX'}\rangle\le 0$.
Therefore, every coordinate of the excess-cost vector $v^{\cX'}$ is nonpositive, which equivalently means that there is an algorithm $\nu^{\cX'}\in\mathfrak A$ satisfying
\begin{align}
    g_x(\nu^{\cX'})\le 0 ~\equiv~ \E_{Y\sim\nu^{\cX'}_x}[\cost(x,Y)] \le \alpha\OPT(x) + \beta 
    \qquad\text{for every }x\in \cX'.
    \label{eq:finite-set-conclusion}
\end{align}
This proves the minimax theorem \Cref{thm:onlineMinimaxFull} for the finite set of inputs $\cX'$. 

\medskip
We now extend this result to the possibly infinite set of instances $\cX$. For each $x\in\cX$, define the set of online algorithms with nonpositive excess-cost on input $x$:
\[
    K_x
    :=
    \{\nu\in\mathfrak A:g_x(\nu)\le 0\}.
\]
 The finite-input conclusion \eqref{eq:finite-set-conclusion} shows that for every finite subset of inputs $\cX' \subseteq\cX$, we have
    $\bigcap_{x\in \cX'}K_x\neq\emptyset.$
Moreover, each $K_x$ is closed, because $g_x$ is continuous and $K_x = g_x^{-1}((-\infty, 0])$.
Since $\mathfrak A$ is compact, a family of closed subsets of $\mathfrak A$ for which every finite subfamily has nonempty intersection must itself have nonempty intersection. Therefore we obtain 
$    \bigcap_{x\in\cX}K_x\neq\emptyset.$
Any $\nu^\star$ in this intersection satisfies $g_x(\nu^\star)\le 0$ for every input $x \in \cX$, or equivalently, $\E_{Y\sim\nu^\star_x}[\cost(x,Y)] \le \alpha\OPT(x) + \beta$. This proves \Cref{thm:onlineMinimaxFull}.

\section{Integral Online Load Balancing} \label[appendix]{sec:integLoadBalancing}
    The principle of posterior matching can also be applied to integral problems, namely online load balancing. To illustrate this, let's consider the online load balancing problem as stated in \Cref{sec:online-load-balancing} but with integral decisions. That is, we require not only $x_t \in \Delta_m$ for all $t$ but also $x_t \in \{0, 1\}^m$. Note that now setting $X\up{t}_t = \E[X^*_t \mid \cF_t]$ for each arrival $t$ as before will generally not be a feasible choice. Instead, we draw $\widehat X_t$ from the same distribution as $X^*_t$, conditioning on $\cF_t$. We again use the standard guess-and-double trick to assume that we have an optimum estimate $\widehat{\OPT}$ such that $\Opt \leq \widehat{\OPT} \leq 2\Opt$ for every scenario, where $\Opt$ is the hindsight optimum objective value.

We will show that this algorithm is $O(p)$-competitive if the objective function is a $p$-supermodular norm \cite{KMS-STOC24}, meaning that 
\[
\| u + v + w \|^p - \| u + v \|^p \geq \| u + w \|^p - \| u \|^p
\]
for any non-negative vectors $u, v, w$. In particular, every $\ell_p$-norm is $p$-supermodular and every symmetric norm can be $O(\log m)$-approximated by an $O(\log m)$-supermodular norm \cite{KMS-STOC24}.

\begin{theorem}
    Let $\|\cdot\|$ be a monotone $p$-supermodular norm on $\mathbb{R}_+^m$. Then, for the integral online load balancing problem with norm $\|\cdot\|$ objective, the posterior-matching algorithm is $\frac{2 p}{\ln 2}$-competitive under correlated stochastic arrivals.
\end{theorem}

\begin{proof}
    We let $L^{\Alg, t}$ and $L^{\Opt, t}$ be the load vectors of the posterior-matching algorithm and the hindsight optimum up to time $t$, respectively, i.e. for each machine $i$,
    \[
    L_i^{\mathrm{Alg},t}=\sum_{s\le t}\widehat X_{si}p_{si} \qquad \text{and} \qquad
L_i^{\mathrm{OPT},t}=\sum_{s\le t}X_{si}^\star p_{si}.
    \]
    Furthermore, let $Y^{\Alg, t} = L^{\Alg, t} - L^{\Alg, t-1}$ and $Y^{\Opt, t} = L^{\Opt, t} - L^{\Opt, t-1}$ be the respective load increases in step $t$. By these definitions, we can write
    \begin{equation}
    \E\left[ \| L^{\Alg, T} \|^p \right] = \sum_{t \in [T]} \E\left[ \| L^{\Alg, t-1} + Y^{\Alg, t} \|^p - \| L^{\Alg, t-1} \|^p  \right].
    \label{eq:integralloadbalancing:telescoping}
    \end{equation}
    Conditioning on $\cF_t$, so the arrivals before step $t$, $Y^{\Alg, t}$ and $Y^{\Opt, t}$ are identically distributed. As $L^{\Alg, t-1}$ is fixed, this implies
    \[
    \E\left[ \| L^{\Alg, t-1} + Y^{\Alg, t} \|^p - \| L^{\Alg, t-1} \|^p  \mid \cF_t \right] = \E\left[ \| L^{\Alg, t-1} + Y^{\Opt, t} \|^p - \| L^{\Alg, t-1} \|^p  \mid \cF_t \right].
    \]
    Taking the expectation over $\cF_t$, we obtain
    \[
    \E\left[ \| L^{\Alg, t-1} + Y^{\Alg, t} \|^p - \| L^{\Alg, t-1} \|^p  \right] = \E\left[ \| L^{\Alg, t-1} + Y^{\Opt, t} \|^p - \| L^{\Alg, t-1} \|^p  \right].
    \]
    Furthermore, we assumed the norm to be $p$-supermodular, implying
    \begin{align*}
    \E\left[ \| L^{\Alg, t-1} + Y^{\Opt, t} \|^p - \| L^{\Alg, t-1} \|^p  \right] & \leq \E\left[ \| L^{\Alg, T} + L^{\Opt, t-1} + Y^{\Opt, t} \|^p - \| L^{\Alg, T} + L^{\Opt, t-1} \|^p  \right] \\
    & = \E\left[ \| L^{\Alg, T} + L^{\Opt, t} \|^p - \| L^{\Alg, T} + L^{\Opt, t-1} \|^p  \right].
    \end{align*}
    Plugging this bound back into \eqref{eq:integralloadbalancing:telescoping}, we obtain
    \begin{align*}
    \E\left[ \| L^{\Alg, T} \|^p \right] & \leq \sum_{t \in [T]} \E\left[ \| L^{\Alg, T} + L^{\Opt, t} \|^p - \| L^{\Alg, T} + L^{\Opt, t-1} \|^p \right] \\
    & = \E\left[ \| L^{\Alg, T} + L^{\Opt, T} \|^p - \| L^{\Alg, T} \|^p \right],
    \end{align*}
    which implies by triangle inequality
    \begin{align*}
    2 \E\left[ \| L^{\Alg, T} \|^p \right] & \leq \E\left[ ( \| L^{\Alg, T} \| + \| L^{\Opt, T} \|)^p \right].
    \end{align*}
    
    Using Minkowski's inequality,
    \[
        2^{1/p} (\E \|L^{\Alg, T}\|^p)^{1/p} \leq  \Big(\E\Big[ \big(\|L^{\Alg, T}\| + \|L^{\Opt, T}\|\big)^p \Big]\Big)^{1/p} \leq (\E[\|L^{\Alg, T}\|^p])^{1/p} + (\E[\|L^{\Opt, T}\|^p)^{1/p} .
    \]
Since Jensen's inequality gives $\E \|L^{\Alg, T}\| \leq (\E \|L^{\Alg, T}\|^p)^{1/p}$, rearranging implies
\[
\E \|L^{\Alg, T}\| ~\leq~ (\E \|L^{\Alg, T}\|^p)^{1/p} ~\leq~ \frac{1}{2^{1/p}-1}(\E[\|L^{\Opt, T}\|^p)^{1/p} ~\leq~ \frac{p}{\ln 2} \widehat{\Opt} ~\leq~ \frac{2 p}{\ln 2} \Opt,
\]
 where the last inequality uses $\|L^{\Opt, T}\|=\OPT \leq \widehat{\Opt} \leq 2\Opt$ almost surely.
\end{proof}

\section{Missing Proofs}

\subsection{Proof of Lemma \ref{lemma:Cgood}} \label{sec:proofCGood}

\subsubsection{Weighted Orlicz Norms} \label{sec:c-good-Orlicz}

We prove Item 1 of Lemma \ref{lemma:Cgood}, namely consider a weighted Orlicz norm $\|\cdot\| = \|\cdot\|_{\Psi,w}$; we will show that it is 2-good. For that, let $X\up{1},\ldots,X\up{T} \in [0,1]^m$ be a martingale such that $\|X\up{T}\| \leq B$, and fix $\e \in (0,1]$. Let 
$M := (\max_{t \in [T]} X\up{t}_i)_{i \in [m]}$ be the coordinate-wise maximum, and let $\epsilon \in (0,1]$. For any value $v \in [0,1]$, we can express $(v - \e)^+ = \int_\e^1 \mathbf{1}_{v \ge \delta}\, \d \delta$, so applying this to each coordinate of the vector $M$ we have, in every scenario, $(M-\e \ones)^+ = \int_\e^1 \ones_{S_\delta}\,\d\delta$, where $S_\delta = \{i \in [m] : M_i \geq \delta\}$ is the set of coordinates where the vector $M$ is at least value $\delta$. Using convexity of norms and Fubini-Tonelli to exchange the order of integration, we get
\begin{align}
    \E\|(M - \epsilon \mathbf{1})^+\| \leq \int_\epsilon^1 \E \|\mathbf{1}_{S_\delta}\|\, \d\delta\, \label{eq:orlicz}
\end{align}
    Our goal will be to show $\E \|\mathbf{1}_{S_\delta}\| \leq \frac{2B}{\delta}$, after which the 2-goodness of $\|\cdot\|$ will directly follow. 

    To this end, recalling the definition of the weighted Orlicz norm $\|x\| = \inf \{\lambda > 0 : \sum_i w_i \cdot \Psi(x_i/\lambda) \le 1\}$, we note that for any set $U \subseteq [m]$ the indicator $\ones_U$ has norm $\|\ones_U\| = \frac{1}{\Psi^{-1}(1/w(U))}$, where $w(U) = \sum_{i \in U} w_i$ and $\Psi^{-1}$ is the generalized inverse function
    \[
    \Psi^{-1}(y) := \sup\{x \geq 0 : \Psi(x) \leq y\};
    \]
    this is because setting $\lambda = \frac{1}{\Psi^{-1}(1/w(U))}$ gives $\Psi((\ones_U)_i /\lambda) = 0$ if $i \notin U$ and we have $\Psi((\ones_U)_i /\lambda) = \Psi(\Psi^{-1}(1/w(U))) \le \frac{1}{w(U)}$, so $\sum_{i \in U} w_i \cdot \Psi(1/\lambda) \le \sum_{i \in U} w_i \cdot \frac{1}{w(U)} = 1$ (and by definition of $\Psi^{-1}$, any decrease in this $\lambda$ makes the sum $> 1$). 
    
    Thus, using this representation with $U = S_\delta$, we get $\|\ones_{S_\delta}\| = \frac{1}{\Psi^{-1}(1/w(S_\delta))}$. Now we claim that $\E\|\mathbf{1}_{S_\delta}\| \leq \frac{2}{\Psi^{-1}(1/\E w(S_\delta))}$, i.e., this representation is almost concave in $w(S_{\delta})$. To see this, note that since $\Psi$ is convex, we have that $\Psi^{-1}$ is concave, and hence the value of $\Psi^{-1}(y)/y$ is decreasing in $y$. Letting $\mu = \E w(S_\delta)$, we can compute
    \begin{align}
    \E\|\mathbf{1}_{S_\delta}\| 
    &= 
    \E\left[\frac{1}{\Psi^{-1}(1/w(S_\delta))}\right] \notag\\
    &= 
    \E\left[\ind_{w(S_\delta) \leq \mu} \cdot \frac{1}{\Psi^{-1}(1/w(S_\delta))}\right]
    +
    \E\left[\ind_{w(S_\delta) > \mu} \cdot \frac{1}{\Psi^{-1}(1/w(S_\delta))}\right] \notag\\
    &\leq
    \frac{1}{\Psi^{-1}(1/\mu)} + \E\left[\ind_{w(S_\delta) > \mu} \cdot w(S_\delta) \cdot \frac{1/\mu}{\Psi^{-1}(1/\mu)}\right] \notag\\
    &=
    \frac{1}{\Psi^{-1}(1/\mu)} + \E\left[\ind_{w(S_\delta) > \mu} \cdot w(S_\delta)\right] \cdot \frac{1/\mu}{\Psi^{-1}(1/\mu)}
    ~\le~ \frac{2}{\Psi^{-1}(1/\mu)}. \label{eq:indOrlicz}
    \end{align}

    With this claim, we can focus on bounding $\E\, w(S_\delta) = \sum_i w_i\Pr[M_i \geq \delta]$. For each $i$, Recall $B$ is an upper bound on $\|X\up{T}\|$. Doob's maximal inequality on the submartingale $(\Psi(X\up{t}_i/B))_{t \in [T]}$ gives
    \[
    \Pr[M_i \geq \delta] \leq \frac{\E \Psi(X\up{T}_i/B)}{\Psi(\delta/B)}.
    \]
    Hence, we compute
    \[
    \E\, w(S_\delta) = \sum_i w_i\,\Pr[M_i \geq \delta] \leq \frac{\sum_i w_i \cdot \E \Psi(X\up{T}_i/B)}{\Psi(\delta/B)} \leq \frac{1}{\Psi(\delta/B)},
    \]
    where the inequality $\sum_i w_i \cdot \E \Psi(X\up{T}_i/B) \leq 1$ follows from  $\|X\up{T}\| \le B$. Applying this bound on \eqref{eq:indOrlicz} we finally obtain the desired bound on $\E \|\ind_{S_\delta}\|$:
    \[
    \E\|\mathbf{1}_{S_\delta}\| ~\leq~ \frac{2}{\Psi^{-1}(1/\E w(S_\delta))} ~\leq~ \frac{2}{\Psi^{-1}(\Psi(\delta/B))} ~\leq~ \frac{2B}{\delta}.
    \]

Employing this bound on \eqref{eq:orlicz} and performing the integration gives
 \[   \E\|(M - \epsilon \mathbf{1})^+\| ~\leq~ \int_\epsilon^1 \frac{2B}{\delta}\, \d\delta ~\leq~ 2B \log(1/\epsilon)\,.
 \]
 This proves that the norm $\|\cdot\|$ is 2-good, and concludes the proof of this item of Lemma \ref{lemma:Cgood}.


\subsubsection{Lov\'asz Extension Norms}

We now prove Item 2 of Lemma \ref{lemma:Cgood}, namely let $\|x\| = L_f(x) = \int_0^\infty f(\{i \in [m] : x_i \ge \lambda\})\,\d \lambda$ be a Lov\'asz extension norm for some  monotone submodular function $f$; we will show that $\|\cdot\|$ is 1-good. 

As in the previous section, let $M := (\max_{t \in [T]} X\up{t}_i)_{i \in [m]}$ and let $S_{\lambda} := \{i \in [m] : M_i \ge \lambda\}$ be the set of coordinates where the maximum is at least $\lambda$. The goodness of the norm $\|\cdot\|$ will follow from the following extension of Doob's  maximal inequality for vectors.

\begin{lemma} \label{lemma:DoobLovasz}
    It holds that
    \[
    \lambda \cdot \E f(S_\lambda) \leq \E\|X\up{T}\|.
    \]
\end{lemma}
\begin{proof}
    For a vector $x \in \R^m$ we use $x_U$ to denote the vector obtained from $x$ by zeroing out all coordinates outside $U$. Also, for a set $U \in [m]$ and coordinate $i \in [m]$ we use the marginal value notation $f(i \mid U) := f(U \cup \{i\}) - f(U)$.

    We will prove the stronger statement that $\lambda \cdot \E f(S_\lambda) \leq \E\|X\up{T}_{S_\lambda}\|$. 
    We first note the following deterministic property of the norm $\|\cdot\|$: for every vector $x \in \Rp^m$, set $V \subseteq [m]$ and $i \in [m] \setminus V$, we have $
    \|x_{V \cup \{i\}}\| - \|x_V\| \ge x_i \cdot f(i \mid V)$:
    \begin{align*}
    \|x_{V \cup \{i\}}\| - \|x_V\| &= \int_0^\infty \Big[f(\{j \in V \cup \{i\} : x_j \ge \lambda\}) - f(\{j \in V : x_j \ge \lambda\}) \Big]\,\d \lambda \notag\\
    & = \int_0^\infty \ind_{x_i \ge \lambda} \cdot f(i \mid \{j \in V : x_j \ge \lambda\})\,\d \lambda \notag\\
    & \ge \int_0^\infty \ind_{x_i \ge \lambda} \cdot f(i \mid V)\,\d \lambda = x_i \cdot f(i \mid V),
    \end{align*}
    where the inequality uses submodularity of $f$.    Now let $i_1, \dots, i_k$ be the elements of $S_\lambda$ in non-decreasing order of the times $t_1 \leq \dots \leq t_k$ at which each respective coordinate $X\up{t}_{i_\ell}$ first hits value at least $\lambda$. Applying the previous claim and taking expectations we get for any $\ell$ 
    \begin{align*}
    \E \|X\up{T}_{i_1,\ldots,i_\ell}\| - \E\|X\up{T}_{i_1,\ldots,i_{\ell-1}}\| \ge \E\Big[X\up{T}_{i_\ell} \cdot f(i_\ell \mid i_1,\ldots,i_{\ell-1}) \Big] & = \E\Big[ \E\Big[X\up{T}_{i_\ell} \cdot f(i_\ell \mid i_1,\ldots,i_{\ell-1}) \,\Big|\, \cF_{t_\ell}\Big] \Big] \\
    &= \E\Big[ \E\Big[X\up{T}_{i_\ell}  \,\Big|\, \cF_{t_\ell} \Big] \cdot f(i_\ell \mid i_1,\ldots,i_{\ell-1}) \Big] \\
    & = \E\Big[ X\up{t_\ell}_{i_\ell} \cdot f(i_\ell \mid i_1,\ldots,i_{\ell-1})\Big] \\
    & \ge \lambda \cdot \E  f(i_\ell \mid i_1,\ldots,i_{\ell-1}),
    \end{align*}
    where the second equation follows from the fact that $\cF_{t_\ell}$ determines the set $\{i_1,\ldots,i_{\ell - 1}\}$, and the last equation from the fact $X\up{t}$ is a martingale. Adding up over all $\ell$ gives
\[
\E\|X^{(T)}_{S_\lambda}\|
=
\E\sum_{\ell=1}^k
\left(
\|X^{(T)}_{\{i_1,\ldots,i_\ell\}}\|
-
\|X^{(T)}_{\{i_1,\ldots,i_{\ell-1}\}}\|
\right)
\ge
\lambda \E f(S_\lambda).
\]
Since the norm is monotone, \(\E\|X^{(T)}_{S_\lambda}\|\le \E\|X^{(T)}\|\), proving the lemma.
\end{proof}

Using this, we can show that the norm $\|\cdot\|$ is $1$-good: for any $\epsilon \in (0,1]$, we have by the definition of the norm
    \[
    \left\|(M - \epsilon\mathbf{1})^+_{i \in [m]}\right\| = \int_0^1 f\Big(\{i : (M_i - \e)^+ \ge \lambda\}\Big)\, \d\lambda \,= \int_\e^1 f(S_\delta)\,\d \delta \,\stackrel{\text{Lem } \ref{lemma:DoobLovasz}}{\leq} \int_\e^1 \frac{\E\|X\up{T}\|}{\delta}\, \d\delta,
    \]
    where we used the change of variables $\delta = \lambda + \e$. Integrating the right-hand side gives $\log(1/\epsilon) \cdot \E\|X\up{T}\| \leq \log(1/\epsilon) \cdot B$. This proves that $\|\cdot\|$ is 1-good as desired, giving Item 2 of Lemma \ref{lemma:Cgood}.

\subsubsection{Monotone Symmetric Norms}

We now prove the last item of Lemma \ref{lemma:Cgood}, namely consider that a monotone symmetric norms are $2(\log m+1)$-good. This relies on the following reduction to ordered norms. Recall that an $\|\cdot\|'$ is an ordered norm over $\Rp^m$ if there are non-negative weights $w_1,\ldots,w_m$ such that for every $x \in \Rp^m$ its norm is $\|x\|' = \sum_i w_i x_i^{\downarrow}$, where $x^{\downarrow}$ denotes the vector $x$ with coordinates sorted in non-increasing order. 

\begin{fact}[\cite{PRS-APPROX23}]\label{fact:approx-sym-norms}
    If $\|\cdot\|$ is a symmetric norm on $\Rp^m$, then there is an ordered norm $\|\cdot\|'$ such that, for any $x \in \Rp^m$, we have
    \[
    \|x\| \leq \|x\|'\leq 2(\log m + 1) \cdot \|x\|.
    \]
\end{fact}

Now take any monotone symmetric norm $\|\cdot\|$, and let $\|\cdot\|'$ be an ordered norm prescribed by the previous lemma. Since every ordered norm is a Lov\'asz extension norm, we know by Item 2 of Lemma \ref{lemma:Cgood} that $\|\cdot\|'$ is 1-good. Moreover, given any martingale $(X \up{t})_{t \in [T]}$ over $\Rp^m$ satisfying $\|X\up{T}\| \leq B$ a.s., we have $\|X\up{T}\|' \leq 2(\log m + 1)  \cdot B$, and hence
    \[
    \E\|(M - \epsilon\mathbf{1})^+\| \leq \E\|(M - \epsilon\mathbf{1})^+\|' \leq \log(1/\epsilon) \cdot 2(\log m + 1) \cdot B,
    \]
and so $\|\cdot\|$ is $2(\log m +1)$-good. This proves the last item of Lemma \ref{lemma:Cgood}. 

\subsection{Proof of \Cref{claim:paging-feasibility}}\label{sec:proofPagingFeasibile}
To complete our proof of \Cref{thm:paging}, we must also verify that our choice of $y^t$ is feasible at each time $t$, i.e. that we have $\sum_i y^t_i \leq k$. To do this, we use a clever approach from \cite{BCLLM-STOC18} in Lemma 3.4 involving a carefully chosen superadditive function $\Phi$: We define the function $\phi : [0,1] \to [0,1]$ by $\phi(x) = \max\{0, \frac{2k-1}{k}x - \frac{k-1}{k}\}$ and $\Phi : \Rp \to \Rp$ by
    \[
    \Phi(x) = \floor{x} + \phi(\{x\}),
    \]
    where $\{x\}$ is the fractional part of $x$. Additionally, define the ``uncapped'' \probmatch{} vector $\hat y^t \in \Rp^n$ by 
        \[
        \hat y^t_i = \begin{cases}
            1 & i = p_t,\\
            0 & i \not\in B(t),\\
            1 - M\up{t}_{i,j(i,t)} & i \in B(t) \setminus \{p_t\}.
        \end{cases}
        \]
    Next we make the following crucial observations:
    \begin{enumerate}
        \item The function $\Phi$ is superadditive.
        \item For each $i$, we have $y^t_i = \Phi\left(\frac{2k}{2k-1} \cdot \hat y^t_i\right)$.
        \item $\sum_{i \in [n]\setminus \{p_t\}} \hat y^t_i \leq k-1$.
    \end{enumerate}
    The first observation is standard (for instance see Lemma 3.4 of \cite{BCLLM-STOC18}), and the second follows from a simple computation: for each $i \in B(t) \setminus \{p_t\}$, we have
    \[
    \Phi\left(\frac{2k}{2k-1} \cdot \hat y^t_i\right) = \max\left\{0,~\min\bigg\{2\hat y^t_i - \frac{k-1}{k},~1\bigg\}\right\} = \max\left\{0,~1 - 2\bigg(M\up{t}_{i,j(i,t)} - \frac{1}{2k}\bigg)^+\right\} = y^t_i.
    \]
    For the third observation, note that
    \[
    \sum_{i \in B(t) \setminus p_t} M\up{t}_{i,j(i,t)} \geq \sum_{i \in B(t) \setminus p_t} X\up{t}_{i,j(i,t)} = \E\Big[\sum_{i \in B(t) \setminus p_t} X^*_{i,j(i,t)} ~\Big|~ \cF_t\Big] \geq |B(t)| - k,
    \]
where the last inequality follows from the feasibility of $X^*$ in our complementary linear program. Hence we have
\[
\sum_{i \in [n] \setminus \{p_t\}} \hat y^t_i = \sum_{i \in B(t) \setminus \{p_t\}} \Big(1 - M\up{t}_{i,j(i,t)}\Big) \leq k-1\,.
\]
    With these observations in mind, we verify the feasibility of $y^t$ by computing
\begin{align*}
\sum_{i \in [n]} y^t_i 
&= 1 + \sum_{i \in [n]\setminus\{p_t\}} y^t_i\\
&= 1 + \sum_{i \in [n]\setminus\{p_t\}} \Phi\left(\frac{2k}{2k-1} \cdot \hat y^t_i\right) &&\text{by Observation 2,}\\
&\leq 1 +  \Phi\Big(\sum_{i \in [n]\setminus\{p_t\}}\frac{2k}{2k-1} \cdot \hat y^t_i\Big)&&\text{by Observation 1,}\\
&\leq 1 +  \Phi\Big(k-1 + \frac{k-1}{2k-1}\Big)&&\text{by Observation 3,}\\
&= 1 + (k-1) = k. && \qedhere
\end{align*}

\begin{small}
\bibliographystyle{alpha}
\bibliography{bib}
\end{small}

\end{document}